%% file: main.tex
\documentclass[1p,authoryear]{elsarticle}

\usepackage{hyperref} 
\usepackage{subcaption}
\usepackage{url}
\usepackage{algorithm}
\usepackage[noend]{algorithmic}
\usepackage{tikz}
\input{macros/notations}

\input{macros/abbreviations}
\input{macros/mplots}

\input{macros/theorems}

\input{macros/mdrawing}
\usepackage{booktabs}
\usepackage{tabularx}
\usepackage{multirow}
\usepackage{bm}
\usepackage{csquotes}

\usepackage{setspace}

\usepackage[colorinlistoftodos]{todonotes}
\makeatletter
\renewcommand{\todo}[2][]{\tikzexternaldisable\@todo[#1]{#2}\tikzexternalenable}
\makeatother

\begin{document}

\begin{frontmatter}

 \title{Learning Reward Machines: A Study in Partially\\Observable Reinforcement Learning\tnoteref{t1}}
 \tnotetext[t1]{This work is an extended version of our previous NeurIPS19 publication \citep{tor-etal-neurips19}.}
 
\author[1,3]{Rodrigo Toro Icarte}
\ead{rntoro@uc.cl}

\author[2]{Ethan Waldie}

\author[2,3]{Toryn Q. Klassen}

\author[4]{Richard Valenzano}

\author[1]{Margarita P. Castro}

\author[2,3]{Sheila A. McIlraith}

\address[1]{Pontificia Universidad Católica de Chile, Vicuña Mackenna 4860, Macul, RM, Chile}
\address[2]{University of Toronto, 214 College St, Toronto, ON, Canada}
\address[3]{Vector Institute, 661 University, Toronto, ON, Canada}
\address[4]{Ryerson University, 350 Victoria St, Toronto, ON, Canada}

\begin{abstract}
  \input{sections/00-abstract}
\end{abstract}


\end{frontmatter}

\input{sections/01-introduction}

\input{sections/02-preliminaries}

\input{sections/03-RMs}

\input{sections/04-LRM}

\input{sections/05-solving_lrm}

\input{sections/06-policy}

\input{sections/07-experiments}

\input{sections/08-discussion}

\input{sections/09-related_work}

\input{sections/10-conclusion}

\section*{Acknowledgements}
We gratefully acknowledge funding from the Natural Sciences and
Engineering Research Council of Canada (NSERC), the Canada CIFAR AI Chairs
Program, Microsoft Research. The first author also acknowledges
funding from ANID (Becas Chile). Resources used in preparing this research
were provided, in part, by the Province of Ontario, the Government of
Canada through CIFAR, and companies sponsoring the Vector Institute for
Artificial Intelligence (\url{vectorinstitute.ai/partners/}). Finally, we
thank the Schwartz Reisman Institute for Technology and Society for
providing a rich multi-disciplinary research environment.

\bibliographystyle{elsarticle-harv}
\bibliography{refs}

\end{document}

%% file: macros/notations.tex
\usepackage{mathtools}
\usepackage{amssymb}
\usepackage{amsmath}
\usepackage{xspace}

\newcommand{\reals}{\mathbb{R}}
\newcommand{\tuple}[1]{\langle{#1}\rangle}

\DeclareMathOperator*{\argmax}{arg\,max}
\newcommand{\defeq}{\triangleq}



%% file: macros/abbreviations.tex
\makeatletter
\DeclareRobustCommand\onedot{\futurelet\@let@token\@onedot}
\def\@onedot{\ifx\@let@token.\else.\null\fi\xspace}

\makeatother

%% file: macros/mplots.tex
\usepackage{tikz}
\usepackage{pgfplots,pgfplotstable}
\usepgfplotslibrary{fillbetween}
\pgfplotsset{compat=1.10}

\newcommand{\percentiles}[7]{
\pgfplotstableread{#1}\datatable
  \addplot [name path=pluserror,draw=none,no markers,forget plot]
    table [x={#2},y expr=\thisrow{#5}] {\datatable};

  \addplot [name path=minuserror,draw=none,no markers,forget plot]
    table [x={#2},y expr=\thisrow{#3}] {\datatable};

  \addplot [forget plot,fill=#6,opacity=#7]
    fill between[on layer={},of=pluserror and minuserror];

  \addplot [#6,thick,no markers,line width=1.3pt]
    table [x={#2},y={#4}] {\datatable};
}

\newcommand{\entry}[2]
{\raisebox{3pt}{\tikz{\draw[#1,line width=1.3pt] (0,0) -- (0.6,0);}} #2}

%% file: macros/theorems.tex
\usepackage{amsthm}
\newtheorem{theorem}{Theorem}

\newtheorem{example}[theorem]{Example}

\newtheorem{definition}[theorem]{Definition}

%% file: macros/mdrawing.tex
\usepackage{tikz}
\usetikzlibrary{shapes.geometric}
\usetikzlibrary{positioning,automata,arrows}
\usepackage{pifont}
\usepackage{fontawesome,wasysym,marvosym}

\newcommand{\coffee}[0]{{\color{black}\Coffeecup}\xspace}

\definecolor{UofTBlue}{RGB}{0,47,101}
\newcommand{\buttonred}[0]{\color{red!70!white}\CIRCLE}
\newcommand{\buttongreen}[0]{\color{green!70!black}\CIRCLE}
\newcommand{\key}[0]{{\color{black}\faKey}\xspace}

\newcommand{\cookie}[0]{\begin{tikzpicture}
    \draw[draw=black,step=1cm,fill=brown,fill opacity=0.8] (0.5,0.5) circle (0.25);
    \draw[step=1cm,fill=black] (0.65,0.5) circle (0.03);
    \draw[step=1cm,fill=black] (0.45,0.65) circle (0.03);
    \draw[step=1cm,fill=black] (0.5,0.4) circle (0.03);
\end{tikzpicture}}

\newcommand{\cookiebutton}[0]{\begin{tikzpicture}
    \draw[draw=black,step=1cm,line width=0.5mm,fill=teal!50!white] (0.5,0.5) circle (0.25);
\end{tikzpicture}}

\newcommand{\haskey}{\resizebox{2.5mm}{!}{\begin{tikzpicture}\node[draw, thick, shape border rotate=90, isosceles triangle, isosceles triangle apex angle=60, fill=black, fill opacity=1.0, node distance=1cm,minimum height=1.5em] at (0,0) {};\end{tikzpicture}}\xspace}
\newcommand{\door}[0]{{\color{black}\faLock}\xspace}

\newcommand{\rzero}[0]{\includegraphics[width=7pt]{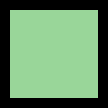}\xspace}
\newcommand{\rone}[0]{\includegraphics[width=7pt]{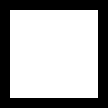}\xspace}
\newcommand{\rtwo}[0]{\includegraphics[width=7pt]{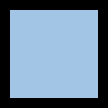}\xspace}
\newcommand{\rthree}[0]{\includegraphics[width=7pt]{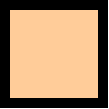}\xspace}
\newcommand{\textcookie}[0]{\includegraphics[width=7pt]{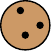}\xspace}
\newcommand{\textcookiebutton}[0]{\includegraphics[width=7pt]{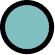}\xspace}
\newcommand{\textcookieeaten}[0]{\includegraphics[width=7pt]{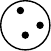}\xspace}

%% file: sections/00-abstract.tex
Reinforcement learning (RL) is a central problem in artificial intelligence. This problem consists of defining artificial agents that can learn optimal behaviour by interacting with an environment -- where the optimal behaviour is defined with respect to a reward signal that the agent seeks to maximize. Reward machines (RMs) provide a structured, automata-based representation of a reward function that enables an RL agent to decompose an RL problem into structured subproblems that can be efficiently learned via off-policy learning. Here we show that RMs can be learned from experience, instead of being specified by the user, and that the resulting problem decomposition can be used to effectively solve partially observable RL problems. We pose the task of learning RMs as a discrete optimization problem where the objective is to find an RM that decomposes the problem into a set of subproblems such that the combination of their optimal memoryless policies is an optimal policy for the original problem. We show the effectiveness of this approach on three partially observable domains, where it significantly outperforms A3C, PPO, and ACER, and discuss its advantages, limitations, and broader potential.\footnote{Our code is available at \url{https://bitbucket.org/RToroIcarte/lrm}.}

%% file: sections/01-introduction.tex
\section{Introduction}
\label{sec:intro}

A fundamental component of human intelligence is our ability to make decisions -- to decide how to act. Indeed, decision making is essential not only to individuals, but to companies, to governments, and to computer-controlled systems that ensure the safe and effective operation of much of our modern infrastructure. Unfortunately, making good decisions can be hard. It can depend on complex inter-relationships between diverse factors, not all of them observable, or well understood. \emph{Reinforcement learning (RL)} endeavours to solve sequential decision-making problems using minimal supervision and minimal prior knowledge. Its goal is to define artificial agents that learn optimal behaviour by interacting with an environment, which may take the form of a simulator or the real world \citep{sutton2018reinforcement}. Every interaction with the environment delivers a reward signal. An RL agent seeks to learn a \emph{policy} (a mapping from observations to actions) that maximizes its expected cumulative reward, improving its policy over time by learning from its past experiences. 

The use of neural networks for function approximation has led to many recent advances in RL. Such deep RL 
methods have allowed agents to learn effective policies in many complex environment including board games \citep{silver2017mastering}, video games \citep{mnih2015human}, and robotic systems \citep{andrychowicz2018learning}. However, RL methods (including deep RL methods) often struggle when the environment is \emph{partially observable}. Indeed, partial observability is one of the main challenges towards applying RL in real-world problems \citep{dulac2019challenges,dulac2021challenges}. This is because agents in such environments usually require some form of memory to learn optimal behaviour \citep{singh1994learning}. Recent approaches for giving memory to an RL agent either rely on recurrent neural networks \citep[e.g.,][]{hausknecht2015deep,mnih2016asynchronous,jaderberg2016reinforcement,wang2016sample,schulman2017proximal}, memory-augmented neural networks \citep[e.g.,][]{oh2016control,khan2017memory,hung2018optimizing}, or external memories that the agent can control using primitive actions \citep[e.g.,][]{littman1993optimization,peshkin1999learning,zhang2016learning,icarte2020act}.

In this work, we show that \emph{reward machines (RMs)} are another useful tool for providing memory in a partially observable environment. RMs were originally conceived to provide a structured, automata-based representation of a reward function \citep{icml2018rms,icarte2020reward,camamacho2019ijcai,DeGiacomo2020transducers}. Exposed structure can be exploited by the \emph{Q-learning for reward machines (QRM)} algorithm \citep{icml2018rms}, which simultaneously learns a separate policy for each state in the RM. QRM has been shown to outperform standard and hierarchical deep RL over a variety of discrete and continuous domains. However, QRM was only defined for fully observable environments. Furthermore, the RMs were handcrafted by a user and then given to the RL agent, thus allowing the agent to exploit the exposed problem substructure. Here, we propose a method for learning an RM directly from experience in a partially observable environment, in a manner that allows the RM to serve as memory for an RL algorithm. 

There are three main contributions of this work. The first is to propose a discrete optimization problem for learning reward machines from experience in a partially observable environment, where the objective is to find a reward machine that makes the problem \emph{as Markovian as possible}. A requirement is that the RM learning method be given a finite set of detectors for properties that serve as the vocabulary for the RM. The model is also fed with traces collected by the agent while exploring the environment. Then, the optimization problem's objective function ranks the reward machines according to how well they predict future observations given the current RM state. 

Our second contribution is to study different methodologies to solve the resulting optimization problem for learning reward machines. In particular, we propose a \emph{mixed integer linear programming (MILP)} model, a \emph{constrained programming (CP)} model, and two \emph{local search (LS)} methods. In our experiments, the best performance was obtained by the local search methods.

Finally, we show how to integrate our models for learning reward machines into the agent-environment interaction loop and show the effectiveness of doing so. By simultaneously learning an RM and a policy for the environment, we are able to significantly outperform several deep RL baselines that use recurrent neural networks as memory in three partially observable domains. We also extend the RM-tailored algorithm \emph{Q-learning for reward machines (QRM)} to the case of partial observability where we see further gains when combined with our RM learning method.

This paper builds upon \cite{tor-etal-neurips19} -- where we originally proposed to formulate the problem of learning an RM as a discrete optimization problem and solved it using tabu search. In this work, we provide further details about this learning pipeline and propose three novel formulations to learn reward machines. These new formulations include a MILP, CP, and LS model. We compare the performance of these models relative to tabu search and found that a local search approach with restarts is consistently better at finding high-quality RMs than tabu search. As a result, the performance of our method improves considerably with respect to the performance reported in our previous publication.

%% file: sections/02-preliminaries.tex
\section{Preliminaries}
\label{sec:preliminaries}
RL agents learn policies from experience. When the problem is fully-observable, the underlying environment model is typically assumed to be a \emph{Markov decision process (MDP)}. An MDP is a tuple $\mathcal{M} = \tuple{S,A,r,p,\gamma,\mu}$, where $S$ is a finite set of \emph{states}, $A$ is a finite set of \emph{actions}, $r: S \times A \rightarrow \reals$ is the \emph{reward function}, $p(s,a,s')$ is the \emph{transition probability distribution}, $\gamma$ is the \emph{discount factor}, and $\mu$ is the \emph{initial state distribution} where $\mu(s_0)$ is the probability that the agent starts in state $s_0 \in S$. In addition, a subset of the states might be labelled as \emph{terminal states}.

At the beginning of an \emph{episode}, the environment is set to an initial state $s_0$, sampled from $\mu$. Then, at time step $t$, the agent observes the current state $s_t \in S$ and executes an action $a_t \in A$. In response, the environment returns the next state $s_{t+1} \sim p(\cdot|s_t,a_t)$ and the immediate reward $r_{t+1} = r(s_t,a_t,s_{t+1})$. The process then repeats from $s_{t+1}$ until potentially reaching a terminal state, when a new episode will begin. 

The agent's goal is to collect as much reward from the environment as possible. To do so, it learns a \emph{policy} $\pi(a|s)$, which is a probability distribution over the actions $a \in A$ given a state $s \in S$. As the agent interacts with the environment, it also improves its policy until (ideally) finding an \emph{optimal policy} $\pi^*$. An optimal policy is a policy that maximizes the expected return received by the agent, which is formally defined as follows:
\begin{equation}
    \pi^* = \argmax_{\pi} \sum_{s \in S} \mu(s) \mathbb{E}_{\pi}\left[\sum_{t=0}^{\infty}\gamma^t r_t \middle| s_0 = s\right]
\end{equation}

\emph{Q-learning} \citep{watkins1992q} is a well-known RL algorithm that uses samples of experience of the form $(s_t,a_t,r_t,s_{t+1})$ to estimate the optimal Q-function $q^*(s,a)$. Here, $q^*(s, a)$ is the expected return of selecting action $a$ in state $s$ and following an optimal policy $\pi^*$ thereafter. During execution, Q-learning maintains the current estimate of the optimal Q-function (i.e., \emph{Q-value}) for each state $s$ and action $a$, $Q(s,a)$. Given a sampled experience $(s_t,a_t,r_t,s_{t+1})$, where $s_{t+1}$ is the state reached after executing action $a_t$ in state $s_t$ and receiving a reward $r_t$, Q-learning updates $Q(s_t,a_t)$ towards $r_t + \gamma \max_{a' \in A}{Q(s_{t+1},a')}$. Given enough experience, the Q-value estimates will converge to the optimal Q-function, and so an optimal policy $\pi^*$ can be computed by always selecting the action $a$ with the highest value of $Q(s,a)$ for each state $s\in S$.

Unfortunately, Q-learning is impractical when solving problems with large state spaces. In such cases, \emph{function approximation} methods are often used. Instead of storing a Q-value for each state-action pair in a table, deep RL methods like DQN \citep{mnih2015human} and DDQN \citep{van2016deep} represent the Q-function as $Q_\theta(s,a)$, where $Q_\theta$ is a neural network whose inputs are features of the state and the outputs are the Q-value estimates for each action $a \in A$. To train the network, mini-batches of experiences $(s,a,r,s')$ are randomly sampled from an \emph{experience replay} buffer and used to minimize the Bellman error. In the case of DQN, this is accomplished by minimizing the square error between $Q_\theta(s,a)$ and the Bellman estimate $r + \gamma \max_{a'}{Q_{\theta'}(s',a')}$. Note that the updates are made with respect to a \emph{target network} with parameters $\theta'$. The parameters $\theta'$ are held fixed when minimizing the square error, but updated to $\theta$ after a certain number of training updates. The role of the target network is to stabilize learning. DDQN follows a similar approach, but the Bellman estimate is computed by selecting the next action $a'$ using $Q_{\theta}$ instead of the target network. This is, $r+\gamma Q_{\theta'}(s',\argmax_{a'}Q_{\theta}(s',a'))$. 

In partially observable problems, the underlying environment model is typically assumed to be a \emph{Partially Observable Markov Decision Process (POMDP)}. A POMDP is a tuple $\mathcal{P_O} = \tuple{S,O,A,r,p,\omega,\gamma,\mu}$, where $S$, $A$, $r$, $p$, $\gamma$, and $\mu$ are defined as in an MDP, $O$ is a finite set of \emph{observations}, and $\omega(o|s)$ is the \emph{observation probability distribution}. Interacting with a POMDP is similar to interacting with an MDP. The environment starts from a sampled initial state $s_0 \sim \mu$. At time step $t$, the agent is in state $s_t \in S$, executes an action $a_t \in A$, receives an immediate reward $r_{t+1}=r(s_t,a_t,s_{t+1})$, and moves to $s_{t+1}$ according to $p(s_{t+1}|s_t,a_t)$. However, the agent does not observe $s_t$ directly. Instead, the agent observes $o_t \in O$, which is linked to $s_t$ via $\omega$, where $\omega(o_{t}|s_{t})$ is the probability of observing $o_{t}$ from state $s_{t}$ \citep{cassandra1994acting}.

RL methods cannot be immediately applied to POMDPs because the transition probabilities and reward function are not necessarily Markovian w.r.t.\ $O$ (though by definition they are w.r.t. $S$). As such, optimal policies may need to consider the complete history $(o_0, a_0, \dots, a_{t-1}, o_t)$ of observations and actions when selecting the next action.\footnote{Technically, the history of interactions should also include the immediate rewards \citep{izadi2005using}, this is $h_t = (o_0, a_0, r_1, \dots, a_{t-1}, r_t, o_t)$. However, we can remove the immediate rewards from the history without loss of generality because it is always possible to include the immediate reward $r_t$ as part of the observation $o_t$.} Several partially observable RL methods use a recurrent neural network to compactly represent the history, and then use a policy gradient method to train it. However, when we do have access to a full POMDP model $\mathcal{P_O}$, then the history can be summarized into a \emph{belief state}. A belief state is a probability distribution $b_t:S \rightarrow [0,1]$ over $S$, such that $b_t(s)$ is the probability that the agent is in state $s \in S$ given the history up to time $t$. The initial belief state is computed using the initial observation $o_0$: $b_0(s) \propto \omega(s,o_0)$ for all $s \in S$. The belief state $b_{t+1}$ is then determined from the previous belief state $b_t$, the executed action $a_t$, and the resulting observation $o_{t+1}$ as follows: 
\begin{equation}
b_{t+1}(s') \propto\omega(s',o_{t+1}) \sum_{s \in S}{p(s,a_t,s')b_t(s)} \ \ \ \text{for all $s' \in S$.}
\label{eq:belief}
\end{equation}

Since the state transitions and reward function are Markovian w.r.t.\ $b_t$, the set of all belief states $B$ can be used to construct the belief MDP $\mathcal{M}_B$, where the states of $\mathcal{M}_B$ are $B$, $A$ are the  actions, the transition probabilities are computed using equation \eqref{eq:belief}, and the reward function $r_b$ is as follows:
\begin{equation}
r_b(b,a) = \sum_{s \in S}{r(s,a)b(s)}.
\label{eq:reward_belief}
\end{equation}
Any optimal policies for $\mathcal{M}_B$ is also optimal for the POMDP \citep{cassandra1994acting}.

%% file: sections/03-RMs.tex
\section{Reward Machines for Partially Observable Environments}
\label{sec:RMsPO}

In this section, we define RMs for the case of partial observability. We use the following problem as a running example to help explain various concepts.

\begin{example}[The cookie domain]
The \emph{cookie domain}, shown in Figure~\ref{fig:d_cookie}, has three rooms connected by a hallway. The agent (purple triangle) can move in the four cardinal directions. There is a button in the orange room that, when pressed, causes a cookie to randomly appear in the green or blue room. The agent receives a reward of $+1$ for reaching (and  thus eating) the cookie and may then go and press the button again. Pressing the button before reaching a cookie will remove the existing cookie and cause a new cookie to randomly appear. There is no cookie at the beginning of the episode. This domain is partially observable since the agent can only see what it is in the room that it currently occupies, as shown in Figure~\ref{fig:d_view}.
\end{example}

\begin{figure}
    \centering
    \begin{subfigure}[b]{.49\textwidth}
        \centering
        \includegraphics[width=0.75\textwidth]{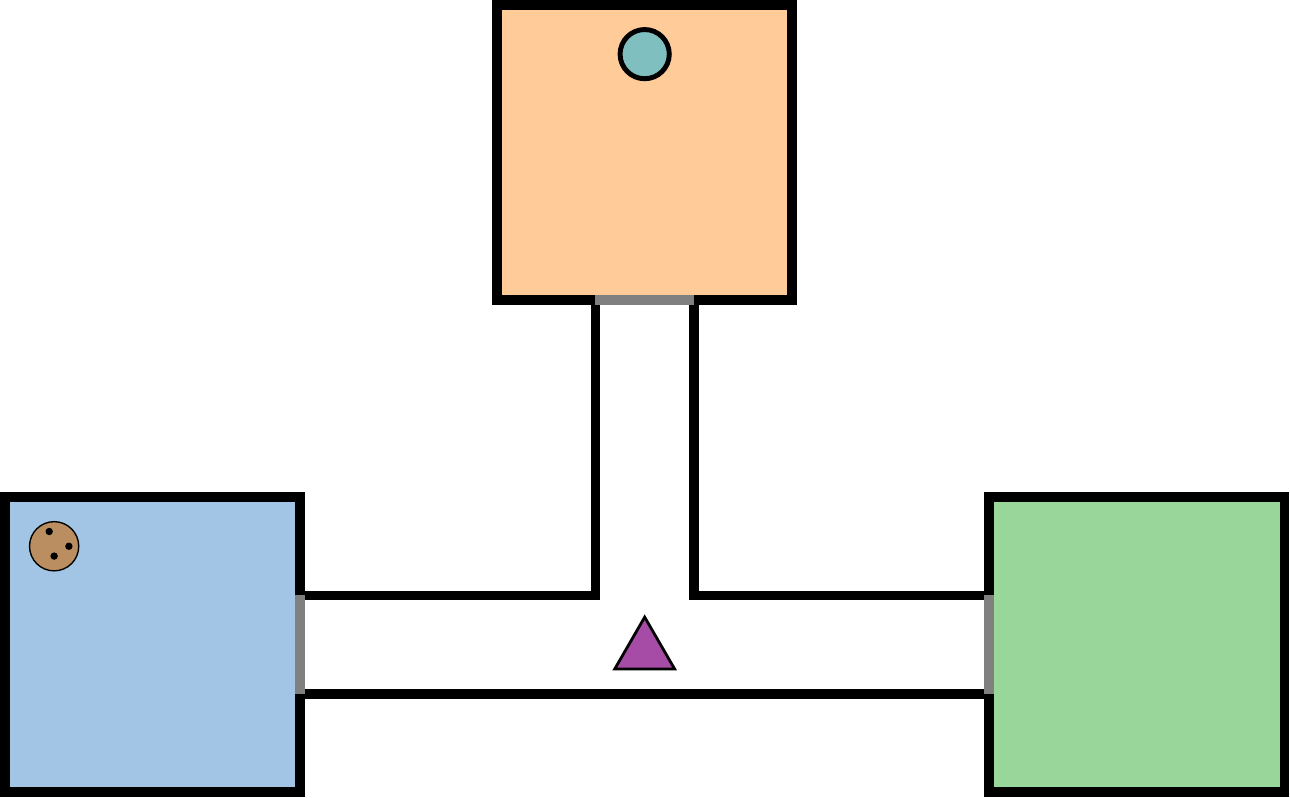}
        \subcaption{Cookie domain.}
        \label{fig:d_cookie}
    \end{subfigure}
    \begin{subfigure}[b]{.49\textwidth}
        \centering
        \includegraphics[width=0.75\textwidth]{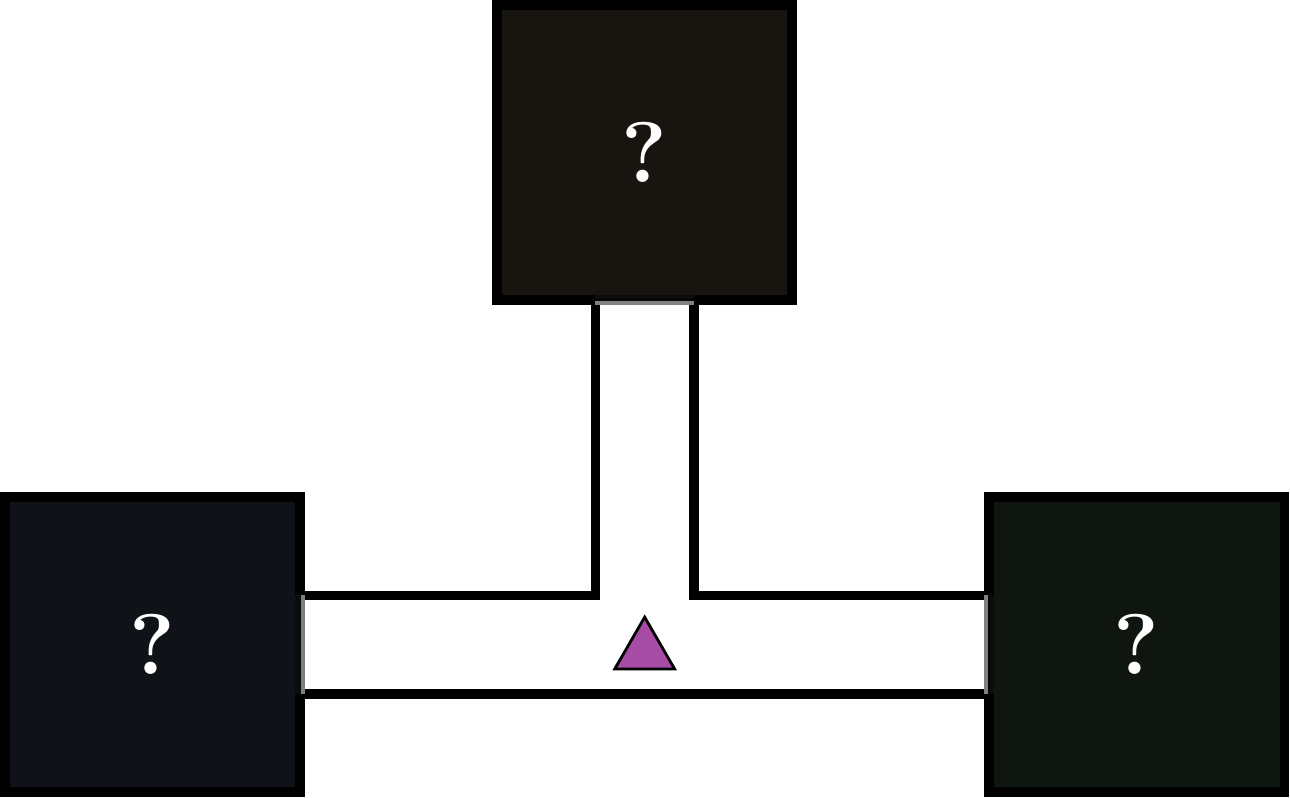}
        \subcaption{Agent's view.}
        \label{fig:d_view}
    \end{subfigure}
    \caption{In the cookie domain, the agent
can only see what is in the current room.}
    \label{fig:domains}
\end{figure}

RMs are finite state machines that are used to encode a reward function \citep{icml2018rms,icarte2020reward}.
They are defined over a set of propositional symbols $\mathcal{P}$ that correspond to a set of high-level features that the agent can detect using a \emph{labelling function} $L:O_{\emptyset} \times A_{\emptyset} \times O \to 2^\mathcal{P}$ where (for any set $X$) $X_{\emptyset} \defeq X \cup \{\emptyset\}$. $L$ assigns truth values to symbols in $\mathcal{P}$ given an environment experience $e = (o,a,o')$ where $o'$ is the observation seen after executing action $a$ when observing $o$. We use $L(\emptyset,\emptyset,o)$ to assign truth values to the initial observation. We call a truth value assignment of $\mathcal{P}$ a \emph{high-level observation} because it provides a high-level view of the low-level environment observations via the labelling function $L$. A formal definition of an RM follows:
\begin{definition}[reward machine]
Given a set of propositional symbols $\mathcal{P}$, a reward machine is a tuple $\mathcal{R}_{\mathcal{P}}=\tuple{U,u_0,\delta_u,\delta_r}$ where $U$ is a finite set of states, $u_0 \in U$ is an initial state, $\delta_u$ is the state-transition function, $\delta_u: U \times 2^\mathcal{P} \to U$, and $\delta_r$ is the reward-transition function, $\delta_r: U \times 2^\mathcal{P} \to \reals$.
\end{definition}

One way of thinking about RMs is that they decompose problems into a set of high-level states $U$ and define transitions using if-like conditions defined by $\delta_u$. These conditions are over a set of binary properties $\mathcal{P}$ that the agent can detect using $L$. For example, in the cookie domain, $\mathcal{P} = \{\textcookie, \textcookieeaten, \textcookiebutton, \rzero, \rone, \rtwo, \rthree\}$. These properties are true in the following situations: $\rzero,\rone,\rtwo$, or $\rthree$ is true if the agent is in a room of that color; $\textcookie$ is true if the agent is in the same room as a cookie; $\textcookiebutton$ is true if the agent pushed the button with its last action; and $\textcookieeaten$ is true if the agent ate a cookie with its last action.

Figure~\ref{fig:rms} shows three possible reward machines for the cookie domain. We note that these three machines define the same reward signal, 1 for eating a cookie and 0 otherwise, but differ in their states and transitions. As a result, they differ with respect to the amount of information about the current POMDP state that can be inferred from the RM state, as we will see below.

Each RM starts in the initial state $u_0$. Edge labels in the figures provide a visual representation of the functions $\delta_u$ and $\delta_r$. For example, label $\tuple{\text{\rzero\textcookieeaten},1}$ between state $u_2$ and $u_0$ in Figure \ref{fig:rm_optimal} represents $\delta_u(u_2, \lbrace \rzero,\textcookieeaten\rbrace)=u_0$ and $\delta_r(u_2, \lbrace \rzero,\textcookieeaten\rbrace)=1$. Intuitively, this means that if the RM is in state $u_2$ and the agent just ate a cookie $\textcookieeaten$ in room $\rzero$, then the agent will receive a reward of $1$ and the RM will transition to $u_0$. Notice that any properties not listed in the label are false (e.g. $\textcookie$ must be false to take the transition labelled $\tuple{\text{\rzero\textcookieeaten},1}$). We also use multiple labels separated by a semicolon (e.g., ``$\tuple{\text{\rtwo},0};\tuple{\text{\rzero\textcookie},0}$'') to describe different conditions for transitioning between the RM states, each with their own associated reward. The label $\tuple{\text{o/w},r}$ (``o/w'' for ``otherwise'') on an edge from $u_i$ to $u_j$ means that transition will be made (and reward $r$ received) if none of the other transitions from $u_i$ can be taken.

Let us illustrate the behaviour of an RM using the one shown in Figure~\ref{fig:rm_perfect}. The RM will stay in $u_0$ until the agent presses the button (causing a cookie to appear), whereupon the RM moves to $u_1$. From $u_1$ the RM may move to $u_2$ or $u_3$ depending on whether the agent finds a cookie when it enters another room. Finally, the RM moves back to $u_0$ from $u_2$ (or $u_3$) when the agent eats a cookie. Note that it is possible to associate meaning with being in RM states. In the example, $u_0$ means that there is no cookie available, $u_1$ means that there is a cookie in some room (either blue or green), $u_2$ means that the cookie is in the green room, and $u_3$ means that the cookie is in the blue room.

When learning a policy for a given RM, one simple technique is to learn a policy $\pi(a|o,u)$ that considers the current observation $o \in O$ and the current RM state $u \in U$ to select action $a \in A$. Interestingly, a partially observable problem might be non-Markovian over $O$, but Markovian over $O \times U$ for some RM $\mathcal{R}_{\mathcal{P}}$. This is the case for the cookie domain with the RM from Figure~\ref{fig:rm_perfect}, for example.

\emph{Q-learning for RMs (QRM)} is another way to learn a policy by exploiting a given RM \citep{icml2018rms}. QRM learns one Q-function $Q_u$ (i.e., policy) per RM state $u \in U$. Then, given any sample experience, the RM can be used to emulate how much reward would have been received had the RM been in any one of its states. Formally, experience $e=(o,a,o')$ can be transformed into a valid experience $(\tuple{o, u}, a, \tuple{o', u'}, r)$ used for updating $Q_u$ for each $u\in U$, where $u' = \delta_u(u, L(e))$ and $r = \delta_r(u, L(e))$. Hence, any off-policy learning method can take advantage of these ``synthetically" generated experiences to update all subpolicies simultaneously. 

When tabular Q-learning is used, QRM is guaranteed to converge to an optimal policy on fully-observable problems \citep{icml2018rms}. 
However, in a partially observable environment, an experience $e$ might be more or less likely depending on the RM state that the agent was in when the experience was collected. For example, experience $e$ might be possible in one RM state $u_i$ but not in RM state $u_j$. Thus, updating the policy for $u_j$ using $e$ as QRM does, would introduce an unwanted bias to $Q_{u_j}$. We will discuss how to (partially) address this problem in Section~\ref{sec:main_loop}.

%% file: sections/04-LRM.tex
\section{Learning Reward Machines from Traces}
\label{sec:approach}

To learn RMs, our overall idea is to search for an RM that can be used as external memory by an agent for solving a partially-observable problem. As input, our method will take a set of high-level propositional symbols $\mathcal{P}$ and a labelling function $L$ that can detect them. Then, the key question is what properties should such an RM have.

\begin{figure}
    \centering
    \begin{subfigure}[b]{.13\textwidth}
        \centering
        \includegraphics[]{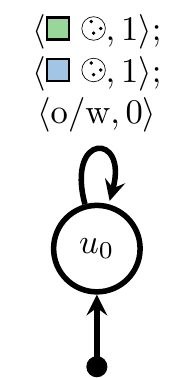}
        \subcaption{Naive RM.}
        \label{fig:rm_naive}
    \end{subfigure}
    \begin{subfigure}[b]{.36\textwidth}
        \centering
        \includegraphics[]{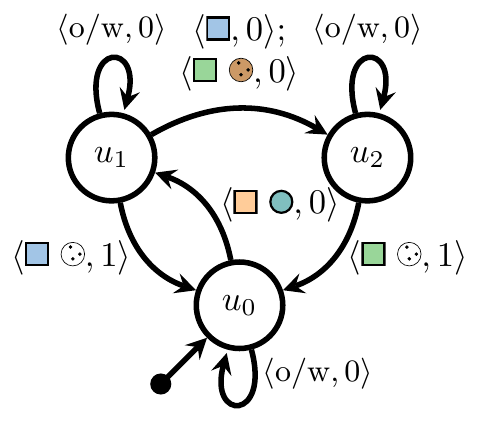}
        \subcaption{``Optimal'' RM.}
        \label{fig:rm_optimal}
    \end{subfigure}
    \begin{subfigure}[b]{.49\textwidth}
        \centering
        \includegraphics[]{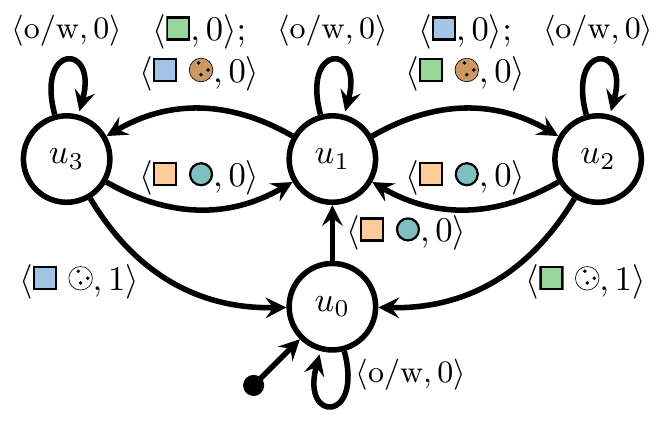}
        \subcaption{Perfect RM.}
        \label{fig:rm_perfect}
    \end{subfigure}
    \caption{Three possible Reward Machines for the Cookie domain.}
    \label{fig:rms}
\end{figure}

Three proposals naturally emerge from the literature. The first comes from the work on learning \emph{finite state machines (FSMs)}  \citep{angluin1983inductive,zeng1993learning,shvo2021interpretable}, which suggests learning the smallest RM that correctly mimics the external reward signal given by the environment, as in \citeauthor{giantamidis2016learning}' method for learning Moore machines \citep{giantamidis2016learning}. Unfortunately, such approaches would learn RMs of limited utility, like the one in Figure~\ref{fig:rm_naive}. This naive RM correctly predicts reward in the cookie domain (i.e., $+1$ for eating a cookie \textcookieeaten, zero otherwise) but provides no memory in support of solving the task.

The second proposal comes from the literature on learning \emph{finite state controllers (FSC)} \citep{meuleau1999learning} and on \emph{model-free RL} methods \citep{sutton2018reinforcement}. This work suggests looking for the RM whose optimal policy receives the most reward. For instance, the RM from Figure~\ref{fig:rm_optimal} is ``optimal'' in this sense. It decomposes the problem into three states. The optimal policy for $u_0$ goes directly to press the button, the optimal policy for $u_1$ goes to the blue room and eats the cookie if present, and the optimal policy for $u_2$ goes to the green room and eats the cookie. Together, these three policies give rise to an optimal policy for the complete problem. This is a desirable property for RMs, but requires computing optimal policies in order to compare the relative quality of RMs, which seems prohibitively expensive. However, we believe that finding ways to efficiently learn ``optimal'' RMs is a promising future work direction.

Finally, the third proposal comes from the literature on \emph{predictive state representations (PSRs)} \citep{littman2002predictive}, \emph{deterministic Markov models (DMMs)} \citep{mahmud2010constructing}, and \emph{model-based RL} \citep{kaelbling1996reinforcement}. This line of research suggests learning the RM that remembers sufficient information about the history to make accurate Markovian predictions about the next observation. For instance, the cookie domain RM shown in Figure~\ref{fig:rm_perfect} is \emph{perfect} w.r.t. this criterion. Intuitively, every transition in the cookie environment is already Markovian except for transitioning from one room to another. Depending on different factors, when entering into the green room there could be a cookie there (or not). The perfect RM is able to encode such information using 4 states: when at $u_0$ the agent knows that there is no cookie, at $u_1$ the agent knows that there is a cookie in the blue or the green room, at $u_2$ the agent knows that there is a cookie in the green room, and at $u_3$ the agent knows that there is a cookie in the blue room. Since keeping track of more information will not result in better predictions, this RM is \emph{perfect}. Below, we develop a theory about perfect RMs and describe an approach for learning them from experience.

\subsection{Perfect Reward Machines: Formal Definition and Properties}

The key insight behind perfect RMs is to use their states $U$ and transitions $\delta_u$ to keep track of relevant past information such that the partially observable environment $\mathcal{P_O}$ becomes Markovian with respect to $O \times U$. This is ensured by the following definition.

\begin{definition}[perfect reward machine]
\label{def:perfect_rm}
An RM $\mathcal{R}_{\mathcal{P}}=\tuple{U,u_0,\delta_u,\delta_r}$ is considered perfect for a POMDP $\mathcal{P_O} = \tuple{S,O,A,r,p,\omega,\gamma,\mu}$ w.r.t.\ a labelling function $L$ iff for every trace $(o_0,a_0,\ldots,o_t,a_t)$ generated by any policy over $\mathcal{P_O}$, the following holds:
$$
\Pr(o_{t+1},r_{t}|o_0,a_0,\ldots,o_t,a_t) = \Pr(o_{t+1},r_{t}|o_t,x_t,a_t),
$$
where $x_0 = u_0$ and $x_t = \delta_u(x_{t-1}, L(o_{t-1},a_{t-1},o_t))$.\footnote{Note that $x_t$ is the RM state that the agent is in at time $t$. We will make use of this notation below.}
\end{definition}

Two important properties follow from Definition~\ref{def:perfect_rm}. First, if the set of reachable belief states $B$ for the POMDP $\mathcal{P_O}$ is finite, then there exists a perfect RM for $\mathcal{P_O}$. Recall that a belief state models the probability of being at any POMDP state given the previous interactions with the environment. When the model of $\mathcal{P_O}$ is known, we can compute the current belief state using the Bayes rule, as discussed in Section~\ref{sec:preliminaries}. This first property states that if there is a finite set of belief states reachable from the initial state, then there exists a perfect RM for that environment.

\begin{theorem}
\label{theo:rm_exists}
For any POMDP $\mathcal{P_O}$ with a finite reachable belief space, there exists a perfect RM for $\mathcal{P_O}$.
\end{theorem}
\begin{proof}
If the reachable belief space $B$ is finite, we can construct an RM that keeps track of the current belief state using one RM state per belief state and emulating their progression using $\delta_u$, and one propositional symbol for every action-observation pair. Thus, the current belief state $b_t$ can be inferred from the last observation, last action, and the current RM state. Hence, the equality from Definition~\ref{def:perfect_rm} holds.
\end{proof}

The second property of perfect RMs is that their optimal policies are also optimal for the POMDP. This means that summarizing the history $h_t = (o_0, a_0, \dots, a_{t-1}, o_t)$ as $\phi(h_t) = \tuple{u_t,o_t}$, where $u_t$ is the current RM state and $o_t$ is the current observation, results in \emph{globally} optimal policies -- i.e., $\pi^*(a_t|h_t) = \pi^*(a_t|u_t,o_t)$.

\begin{theorem}
\label{theo:rm_policy}
Let $\mathcal{R}_{\mathcal{P}}$ be a perfect RM for a POMDP $\mathcal{P_O}$, then any optimal policy for $\mathcal{R}_{\mathcal{P}}$ w.r.t.\ the environmental reward is also optimal for $\mathcal{P_O}$.
\end{theorem}
\begin{proof}
As the next observation and immediate reward probabilities can be predicted from $O \times U \times A$, a perfect RM for $\mathcal{P_O}$ also models the belief MDP $\mathcal{M}_B$ for $\mathcal{P_O}$. As such, optimal policies over $O \times U$ are equivalent to optimal policies over $\mathcal{M}_B$, which are optimal over $\mathcal{P_O}$ \citep{cassandra1994acting}.
\end{proof}

\subsection{Perfect Reward Machines: How to Learn Them}
\label{sec:learn_rms}

We now consider the problem of learning a perfect RM from traces, assuming one exists w.r.t.\ the given labelling function $L$. Recall that a perfect RM transforms the original problem into a Markovian problem over $O \times U$. Hence, we should prefer RMs that accurately predict the next observation $o'$ and the immediate reward $r$ from the current observation $o$, RM state $u$, and action $a$. This might be achieved by collecting a training set of traces from the environment, fitting a predictive model for $\Pr(o',r|o,u,a)$, and picking the RM that makes better predictions. However, this approach can be very expensive, especially considering that the observations might be images.

Instead, we propose an alternative that focuses on a necessary condition for a perfect RM: the RM must predict what is \emph{possible} and \emph{impossible} in the environment at the abstract level given by the labelling function. E.g., it is impossible to be at $u_3$ in the RM from Figure~\ref{fig:rm_perfect} and make the high-level observation $\{\rzero,\textcookie\}$, because the RM reaches $u_3$ only if the cookie was seen in the blue room (or not to be in the green room) and it leaves $u_3$ as soon as the agent eats the cookie (or presses the button).

This idea is formalized in the optimization model \ref{lrm:problem}. Let $\mathcal{T}=\{\mathcal{T}_0, \ldots, \mathcal{T}_n\}$ be a set of traces, where each trace $\mathcal{T}_i$ is a sequence of observations, actions, and rewards:
\begin{equation}
\mathcal{T}_i=(o_{i,0},a_{i,0},r_{i,1}, \ldots, a_{i,t_i-1},r_{i,t_i},o_{i,t_i}).
\end{equation}
We now look for an RM $\tuple{U,u_0,\delta_u,\delta_r}$ that can be used to predict $L(e_{i,t+1})$ from $L(e_{i,t})$ and the current RM state $x_{i,t}$, where $e_{i,t+1}$ is the experience $(o_{i,t},a_{i,t},o_{i,t+1})$ and we define $e_{i,0}$ as $(\emptyset,\emptyset,o_{i,0})$. The model parameters are the set of traces $\mathcal{T}$, the set of propositional symbols $\mathcal{P}$, the labelling function $L$, and a maximum number of states in the RM $u_\text{max}$. The model also uses the sets $I = \{0 \ldots n\}$, $T_i = \{0 \ldots t_i - 1\}$, and $\Sigma$, where $I$ contains the indices of the traces, $T_i$ their time steps, and $\Sigma$ contains all the high-level observations that appear in $\mathcal{T}$. The model has two auxiliary variables $x_{i,t}$ and $N_{u,\sigma}$. Variable $x_{i,t} \in U$ represents the state of the RM after observing trace $\mathcal{T}_i$ up to time $t$. Variable $N_{u,\sigma} \subseteq 2^{\Sigma}$ is the set of all the next high-level observations seen from the RM state $u$ and the high-level observations $\sigma$ in $\mathcal{T}$. In other words, $\sigma' \in N_{u,\sigma}$ iff $u = x_{i,t}$, $\sigma = L(e_{i,t})$, and $\sigma' = L(e_{i,t+1})$ for some trace $\mathcal{T}_i$ and time $t$. 
\input{src/model.tex}%
Constraints \eqref{lrm:properRM} and \eqref{lrm:maxU} ensure that we find a well-formed RM over $\mathcal{P}$ with at most $u_\text{max}$ states. Constraints \eqref{lrm:xInU}, \eqref{lrm:x0ISu0}, and \eqref{lrm:progXs} ensure that $x_{i,t}$ is equal to the current state of the RM, starting from $u_0$ and following $\delta_u$. Constraints \eqref{lrm:Ndef} and \eqref{lrm:NclosedWorld} ensure that the sets $N_{u,\sigma}$ contain every $L(e_{i,t+1})$ that has been seen right after $\sigma$ and $u$ in $\mathcal{T}$. The objective function comes from maximizing the log-likelihood for predicting $L(e_{i,t+1})$ using a uniform distribution over all the possible options given by $N_{u,\sigma}$.

A key property of this formulation is that any perfect RM is optimal w.r.t. the objective function in \ref{lrm:problem} when the number of traces tends to infinity:
\begin{theorem}
When the set of training traces (and their lengths) tends to infinity and is collected by a policy such that $\pi(a|o)> \epsilon$ for all $o \in O$ and $a \in A$, any perfect RM with respect to $L$ and at most $u_\text{max}$ states will be an optimal solution to the formulation \ref{lrm:problem}.
\end{theorem}
\begin{proof}
In the limit, $\sigma' \in N_{u,\sigma}$ if and only if the probability of observing $\sigma'$ after executing an action from the RM state $u$ while observing $\sigma$ is non-zero. In particular, for all $i\in I$ and $t \in T$, the cardinality of $N_{x_{i,t},L(e_{i,t})}$ will be minimal for a perfect RM. This property follows from the fact that perfect RMs make perfect predictions for the next observation $o'$ given $o$, $u$, and $a$. Therefore, as we minimize the sum over $\log(|N_{x_{i,t},L(e_{i,t})}|)$, the objective value for any perfect RM must be minimal.
\end{proof}

Finally, note that the definition of a perfect RM does not impose conditions over the rewards associated with the RM (i.e., $\delta_r$). This is why $\delta_r$ is a free variable in the model \ref{lrm:problem}. However, in order to apply methods that exploit RM structure (such as QRM), we still need $\delta_r$ to model the external reward signals given by the environment. To do so, we estimate $\delta_r(u,\sigma)$ using its empirical expectation over $\mathcal{T}$ -- as commonly done when constructing belief MDPs \citep{cassandra1994acting}. 
Formally,
\begin{equation}
\delta_r(u,\sigma) = \frac{\sum_{i \in I, t \in T_i} r_{i,t+1} 1_{u=x_{i,t} \wedge \sigma=L(e_{t+1})}}{\sum_{i \in I, t \in T_i} 1_{u=x_{i,t} \wedge \sigma=L(e_{t+1})} + \epsilon} \ \ \ \text{for all $u \in U$ and $l \in 2^\mathcal{P}$,}
\label{lrm:reward}
\end{equation}
where $1_c = 1$ if condition $c$ holds (zero otherwise) and $\epsilon$ is a small constant.

%% file: src/model.tex
\begingroup
\allowdisplaybreaks
\begin{align}
\underset{\tuple{U,u_0,\delta_u,\delta_r}}{\text{minimize}}\;& \sum_{i \in I} \sum_{t \in T_i} \log(|N_{x_{i,t},L(e_{i,t})}|) \tag{\texttt{LRM}} \label{lrm:problem} \\
s.t.\; & \tuple{U,u_0,\delta_u,\delta_r} \in \mathcal{R}_{\mathcal{P}} \label{lrm:properRM} \\
& |U| \leq u_\text{max} \label{lrm:maxU} \\
& x_{i,t} \in U & \forall  i \in I, t \in T_i \cup \{t_i\}  \label{lrm:xInU}\\
& x_{i,0} = u_0 & \forall  i \in I  \label{lrm:x0ISu0}\\
& x_{i,t+1} = \delta_u(x_{i,t}, L(e_{i,t+1})) & \forall  i \in I, t \in T_i  \label{lrm:progXs}\\
& N_{u,\sigma} \subseteq 2^{\Sigma} & \forall  u \in U, \sigma \in \Sigma  \label{lrm:Ndef}\\
& L(e_{i,t+1}) \in N_{x_{i,t},L(e_{i,t})} & \forall i \in I, t \in T_i \label{lrm:NclosedWorld}
\end{align}%
\endgroup%

%% file: sections/05-solving_lrm.tex
\section{Searching for a Perfect Reward Machine}
\label{sec:lrm_models}

We now describe different approaches to solve \ref{lrm:problem}. These include a \emph{mixed integer linear programming (MILP)} model, a \emph{constrained programming (CP)} model, and two \emph{local search (LS)} models. All our models are guaranteed to find optimal solutions given sufficient resources. But first, let us introduce two preprocessing steps over the training traces that our models use: \emph{trace compression} and \emph{prefix trees (PTs)}. 

\subsection{Trace Compression}

Recall that \ref{lrm:problem} works over the high-level observations $ \sigma_t$ given by the labelling function $L(e_t)$, where $e_t$ represents the experience $(o_t,a_t,o_{t+1})$ at time $t$. Thus, the first step is to transform each trace $\mathcal{T}_i$ from being a sequence of interactions with the environment into a sequence of truth value assignments of $\mathcal{P}$ via $L$:
$$
\mathcal{T}_i = (o_{i,0},a_{i,0},r_{i,1}, \ldots, a_{i,t_i-1},r_{i,t_i},o_{i,t_i}) \xRightarrow[]{\sigma_t=L(e_t)} \tau_i = (\sigma_{i,0},\sigma_{i,1},\ldots,\sigma_{i,t_i}).
$$

In the abstract space given by $\mathcal{P}$ and $L$, it is usually the case that the same high-level observation is seen many times in a row. For instance, a typical high-level trace in the cookie domain would look as follows:
$$
\tau = (\rone, \rone, \rone, \rone, \rthree, \rthree, \rthree\textcookiebutton, \rthree, \rthree, \rone, \rone, \rone, \rone, \rone, \rone, \rone, \rtwo, \rone, \rone, \rone, \rone, \rone, \rone, \rone, \rone, \rzero\textcookie).
$$
This trace indicates that the agent was first at the hallway ($\rone$), and then it moved to the orange room ($\rthree$) and pressed the button ($\rthree\textcookiebutton$). After that, it returned to the hallway ($\rone$), noticed that there was no cookie in the blue room ($\rtwo$), came back to the hallway ($\rone$), and finally found a cookie in the green room ($\rzero\textcookie$). Let us assume that the trace ended there for simplicity. As you can see, the most informative moments are when the high-level observations change. Indeed, knowing that the agent stayed at the hallway during 4, 7, or 8 steps is not particularly useful in this case. This suggests that we could potentially compress the training trace (without losing relevant information) by removing duplicated high-level observations that appear consecutively in the trace. For instance, the previous trace will look as follows after being compressed:
$$
\tau_c = (\rone, \rthree, \rthree\textcookiebutton,  \rthree, \rone, \rtwo, \rone, \rzero\textcookie).
$$

Compressing the traces is an optional preprocessing step that has advantages and disadvantages. The advantage is that it considerably improves the quality of the RMs that the models find given a fixed computational budget. Intuitively, any model has to go over the traces to evaluate how good an RM is. Compressing the traces reduces the computational cost of doing so. The disadvantage is that, by compressing the traces, we are assuming that observing two or more times the same high-level observation consecutively does not provide further information about the current POMDP state. If that is the case, then we can compress the traces and the model will still find optimal solutions for \ref{lrm:problem}. If that is not the case, then compressing the traces might help the models to find high-quality RMs faster but they will not find optimal solutions.

Two inconsistencies might arise if we compress the traces. First, note that we are learning RMs using compressed traces but then testing them over uncompressed traces. This is a problem because the optimal solution for \ref{lrm:problem} will exploit the fact that no training trace has the same high-level observation twice in a row. However, when the agent uses the learned RM in the environment (i.e., at test time), it might encounter the same high-level observation many times in a row and, thus, update the RM state in ways that were unintended by \ref{lrm:problem}. For that reason, if we do compress the traces, we have to include the following additional constraint to \ref{lrm:problem}:
\begingroup
\allowdisplaybreaks
\begin{align}
\left[\delta_u(u',\sigma) = u\right] & \;\Rightarrow\; \left[\delta_u(u,\sigma) = u\right] & \forall  u, u' \in U, \sigma \in \Sigma  \label{lrm:compress}
\end{align}%
\endgroup%
This constraint ensures that the learned reward machine will not enter and leave a state $u$ using the same high-level observation $\sigma$. 

The second inconsistency relates to constraint \eqref{lrm:progXs}. According to that constraint, the second high-level observation of the trace is used to progress the initial state in the reward machine (and not the first one). Indeed, \ref{lrm:problem} dictates that $x_0 = u_0$ and $x_1 = \delta_u(u_0,\rthree)$ for $\tau_c$. However, if we were processing the uncompressed version of $\tau_c$ (i.e., $\tau$), then the initial RM state would have been updated using $x_1 = \delta_u(u_0,\rone)$ instead of $x_1 = \delta_u(u_0,\rthree)$ because $\rone$ appears many times in a row at the beginning of $\tau$. We can solve this inconsistency by not compressing the first high-level observation of the trace. This is, to always include the first high-level observation and only compress from the second high-level observation forward. In our example, the proper manner of compressing $\tau$ would be as follows:
$$
\tau_c' = (\rone, \rone, \rthree, \rthree\textcookiebutton,  \rthree, \rone, \rtwo, \rone, \rzero\textcookie).
$$

\subsection{Prefix Trees (PTs)}

Regardless of whether we compress or not the training traces, from now on we will be working with a set of $n+1$ traces $\mathcal{T} = \{\tau_0,\ldots,\tau_n\}$, where each trace is composed of a sequence of high-level observations: $\tau_i = (\sigma_{i,0},\ldots,\sigma_{i,t_i})$. Then, \ref{lrm:problem} defines an independent variable $x_{i,t}$ which models the current RM state at time step $t$ given the trace $\tau_i$ for all $i \in I$ and $t \in T_i$. However, doing so does not exploit the fact that there exist large sets of $x_{i,t}$ variables whose values are equivalent for any reward machine. These are cases where the prefix of two traces $\tau_i$ and $\tau_j$ are identical up to some time step $t$. Indeed, we know that variables $x_{i,t}=x_{j,t}$ if $\sigma_{i,t'} = \sigma_{j,t'}$ for all $t'<t$ because the transitions of a reward machine are deterministic. We can compactly capture this information using \emph{prefix trees (PTs)} \citep{de2010grammatical}.

\begin{figure}
    \centering
    \includegraphics[]{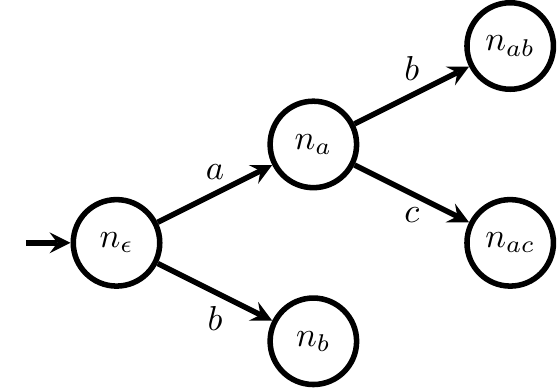}
    \caption{A PT for $\{(b),(ab),(ac)\}$. We took this example from \cite{giantamidis2016learning}.}
    \label{fig:pt_example}
\end{figure}

PTs merge all the training traces into one large tree, where each trace becomes a branch in this tree. As an example, Figure~\ref{fig:pt_example} shows the PT for the following set of traces: $\tau_1 = (b)$, $\tau_2 = (a,a)$, and $\tau_3 = (a,b)$. Each node in a PT represents a prefix that appears in one or more training traces. In the example, there are 5 nodes corresponding to the 5 possible prefixes in $\{\tau_1, \tau_2, \tau_3\}$: $\epsilon$, $(a)$, $(b)$, $(a,b)$, and $(a,c)$, where $\epsilon$ represents the empty trace. Thus, instead of assigning RM states to each variable $x_{i,t}$, our models assign RM states to each node in the PT. That way, these models are forced to assign the same sequence of RM states to all traces as long as their prefixes are identical. For instance, assigning an RM state to node $n_a$ in the PT will automatically assign the same RM state to $x_{2,1}$ and $x_{3,1}$ because $n_a$ represents the current RM state after any trace observes $a$ at the initial time step and, in this case, `$a$' is the first observation of $\tau_2$ and $\tau_3$.

\begin{algorithm}[tb]
\caption{Converting Training Data into Prefix Trees}
\label{fig:algorithm_PT}
\input{src/generate_PT}
\end{algorithm} 

Algorithm~\ref{fig:algorithm_PT} shows the pseudo-code to constructing a PT from a given set of traces $\{\tau_0,\ldots,\tau_n\}$. Starting at the root node, the code goes over all the training traces and adds them as branches of the tree (lines 7-8). We also count how many training traces pass through each node in the tree (line 6). This information helps us defining the right weights to penalize predictions in the objective function when reformulating \ref{lrm:problem} from assigning RM states to time steps to assigning RM states to nodes in the PT.

\subsection{Solving LRM via Mixed Integer Linear Programming (MILP)}

In this section, we present a MILP model for \ref{lrm:problem}. MILP solvers are able to solve optimization problems with linear constraints and objective functions. They can also handle continuous and discrete variables. MILP solvers are the state of the art for solving a wide range of discrete optimization problems and they are guaranteed to find optimal solutions given sufficient resources \citep{junger200950}. 

Our MILP model receives as input $u_{\max}$ and the PT built using the training traces $\mathcal{T} = \{\tau_0,\ldots,\tau_n\}$. Note that the traces in $\mathcal{T}$ might or might not be compressed. We use the following notation to refer to the different elements in the PT: $n_{\text{root}}$ is its root node, $S$ is a set containing all the nodes in PT except for $n_{\text{root}}$, $S_{\text{in}} \subset S$ is a subset of $S$ that only contains the inner nodes of PT, $w_n$ is the number of training traces that pass through node $n$, $p(n)$ is the parent of node $n$, $o(n)$ is the high-level observation associated with the edge between nodes $p(n)$ and $n$, and $C(n)$ is the set of children of node $n$. Also, this model assumes that the set $\Sigma$ contains all the high-level observations that appear in $\mathcal{T}$, $|U| = u_{\max}$, and $K = 2^{|\Sigma|}$.

The idea behind our MILP model is to assign RM states to each node in the PT. Then, we look for an assignment that (i) can be produced by a deterministic machine and (ii) optimizes the same objective as \ref{lrm:problem}. To achieve this, we use the following decision variables. Variable $x_{n,u} \in \{0,1\}$ indicates if node $n \in S \cup \{n_{\text{root}}\}$ is assigned the RM state $u \in U$. Variable $d_{u,\sigma,u'} \in \{0,1\}$ represents the possible transition from state $u \in U$ to state $u' \in U$ given observation $\sigma\in \Sigma$ in the RM. Formally, this means that $d_{u,\sigma,u'} = 1$ iff $\delta_u(u,\sigma) = u'$. Variable $p_{u,\sigma,\sigma'} \in \{0,1\}$ indicates if $\sigma' \in \Sigma$ is a possible next observation at RM state $u \in U$ when observing $\sigma \in \Sigma$, where $p_{u,\sigma,\sigma'} = 1$ iff $\sigma' \in N_{u,\sigma}$. Variable $y_{u,\sigma,m} \in \{0,1\}$ represents the cardinality of $N_{u,\sigma}$, meaning that $y_{u,\sigma,m}=1$ iff $|N_{u,\sigma}| = m$. Lastly, variables $z_{n}$ represent the log-likelihood cost for the predictions associated with node $n \in S$. The full model is then as follows:
\input{src/milp}
The objective function of \ref{mip:problem} is a sum over the prediction costs at each node in the tree weighted by how many traces pass through that node. Constraint \eqref{mip:cost} models the log-likelihood cost for each node in the tree. Constraints \eqref{mip:oneCount} and \eqref{mip:predictionCount} compute the cardinality of  $N_{u,l}$. Constraint \eqref{mip:prediction} defines the possible predictions given an RM state and high-level observation. Constraint \eqref{mip:oneJump} enforces that for each RM state a high-level observation can lead to exactly one other RM state. Constraint \eqref{mip:oneState} enforces that exactly one RM state is assigned to each node in the tree. Constraint \eqref{mip:firstSate} assigns the initial RM state to the root node, and constraint \eqref{mip:stateTransition} enforces that there exists a deterministic $\delta_u$ that can produce the assignment of RM states to tree nodes. Constraints \eqref{mip:domx}-\eqref{mip:domz} correspond to the variables' domains. Finally, we note that constraint \eqref{mip:optionalInOut} is an optional constraint that should be included only if the training traces were compressed. This constraint enforces that the RM state does not change after observing the same high-level observation two times consecutively. 

\subsection{Solving LRM via Constrained Programming (CP)}

CP is another technique for solving discrete optimization problems. CP is less restrictive than MILP in the type of variables and constraints that it can handle. For instance, our MILP model had to include many auxiliary decision variables (e.g., $x_{n,u}$, $p_{u,\sigma,\sigma'}$, $y_{u,\sigma,m}$, and $z_n$) in order to linearize different aspects of \ref{lrm:problem}. In contrast, our CP model only uses one set of decision variables $d_{u,\sigma}$ to model $\delta_u(u,\sigma)$ and all other elements from \ref{lrm:problem} are defined w.r.t.\ those variables. In addition, CP solvers are also guaranteed to find optimal solutions given enough resources \citep{rossi2006handbook}.

Our CP model receives the same inputs as our MILP model, including the PT constructed using the (potentially compressed) training traces $\mathcal{T}$. Recall that the different elements in the PT are referred as follows: $n_{\text{root}}$ is the root node, $S$ is the set containing all nodes but $n_{\text{root}}$, $w_n$ is the number of training traces that pass through node $n$, $p(n)$ is the parent of node $n$, $o(n)$ is the high-level observation between $p(n)$ and $n$, and $C(n)$ is the set of children of node $n$. The model also uses the set $U = \{0...u_{\max}-1\}$, the set $\Sigma$ with all the high-level observations in $\mathcal{T}$, and the set $S(\sigma,\sigma')$ containing all the nodes where $\sigma'$ is observed immediately after observing $\sigma$: 
$$
S(\sigma,\sigma') = \{n \;|\; n \in S, n' \in C(n), \sigma=o(n), \sigma'=o(n')\}.
$$

As we previously mentioned, the only decision variables are $d_{u,\sigma} \in U$ for all $u \in U$ and $\sigma \in \Sigma$. Note that $d_{u,\sigma}$ is an integer variable that goes from $0$ to $u_{\max} - 1$ and it models the output of $\delta_u(u,\sigma)$ -- i.e., if the RM state is $u$ and the agent observes $\sigma$, then the next RM state will be $d_{u,\sigma}$. With that, the complete model is as follows:
{%
\begingroup%
\allowdisplaybreaks
\begin{align}
\min\; & \sum_{n \in S}w_n \cdot \texttt{log}(y_{x_n,o(n)}) \tag{\texttt{CP}} \label{cp:problem}\\
s.t.\; 
& y_{u,\sigma} \doteq \sum_{\sigma' \in \Sigma} p_{u,\sigma,\sigma'} & \forall u \in U, \sigma \in \Sigma \label{cp:oneCount}\\
& p_{u,\sigma,\sigma'} \doteq \texttt{logical\_or}\left(\{x_n = u| n \in S(\sigma, \sigma')\}\right) & \forall u \in U, \sigma,\sigma' \in \Sigma \label{cp:prediction} \\
& x_n \doteq d_{u,o(n)} & \forall n\in S, u = x_{p(n)} \label{cp:stateTransition}\\
& x_n \doteq 0 & n = n_{\text{root}} \label{cp:firstSate}\\
& \texttt{if\_then}(d_{u,\sigma}=u', d_{u',\sigma}=u') & \forall u,u \in U, \sigma \in \Sigma  \label{cp:optionalInOut} \\
& d_{u,\sigma} \in U &\forall u \in U, \sigma \in \Sigma \label{cp:domd}
\end{align}
\endgroup%
}%
This model uses the formalism and global constraints available in IBM ILOG CP Optimizer \citep{CPOManual}. Its only decision variables are $d_{u,\sigma}$, defined in constraint \eqref{cp:domd}. Note that the domain of $d_{u,\sigma} \in U$ forces the RM to be deterministic, since there is exactly one possible transition for each $u \in U$ and $\sigma \in \Sigma$. Using $d_{u,\sigma}$, the model defines three auxiliary \emph{CP expressions} in equations \eqref{cp:oneCount}-\eqref{cp:firstSate}. Expression $y_{u,\sigma}$ represents the cardinality of the prediction set $N_{u,\sigma}$ after observing $\sigma \in \Sigma$ from the RM state $u \in U$. Expression $p_{u,\sigma,\sigma'}$ is one if and only if it is possible to observe $\sigma' \in \Sigma$ from the RM state $u \in U$ after observing $\sigma \in \Sigma$ (and zero otherwise). This expression uses a logical OR constraint, $\texttt{logical\_or}(Z)$ which returns 1 iff at least one element $z \in Z$ is true. Expression $x_n \in U$ indicates that the RM state assigned to the tree node $n$. Finally, the objective function is a weighted sum over the prediction errors and constraint \eqref{cp:optionalInOut} is an optional constraint used if the traces were compressed. This constraint enforces a self loop $\delta_u(u',\sigma) = u'$ if $\delta_u(u,\sigma) = u'$ for some $u \in U$.

\subsection{Solving LRM via Local Search (LS) and Tabu Search (TS)}

Finally, here we present our local search methods. We note that MILP and CP are known as \emph{exact methods}. They incrementally construct a search tree where each branch represents a feasible solution to the problem and use different relaxation and propagation rules to prune this tree as much as possible. This approach allows them to find optimal solutions and prove that those solutions are optimal for small to medium size problems. However, they struggle when facing large problems because the size of the tree grows exponentially with the number of variables and computing relaxations and propagation rules become more expensive with the number of constraints.

When solving large scale problems, the best results are often obtained by \emph{heuristic} methods. Heuristic methods propose polynomial-time approximations to solve NP-hard problems. They favor finding good solutions over providing strong optimality guarantees. Here, we explore two local search methods \citep{aarts2003local}.

\begin{figure}
    \centering
    \includegraphics[]{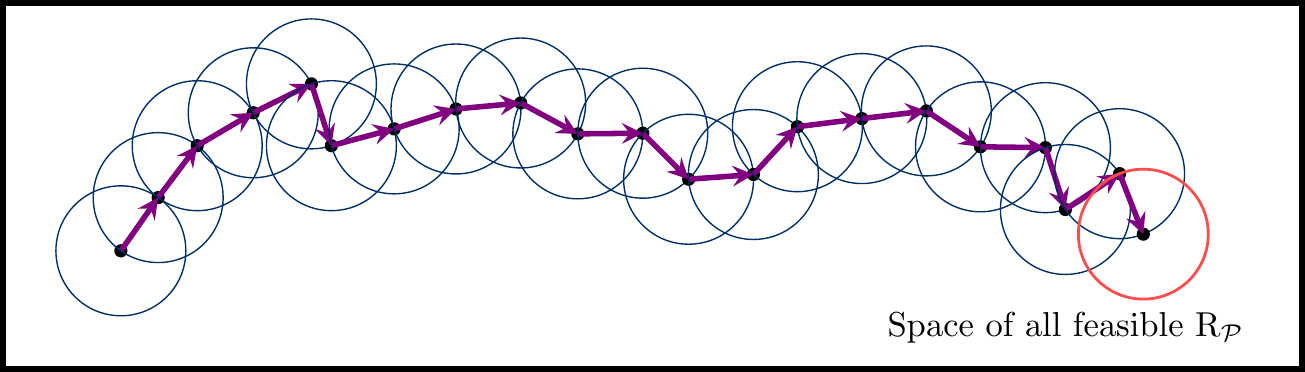}
    \caption{Local search approaches start from some feasible solution and iteratively move to the best solution within its neighbourhood. This process repeats until reaching a locally optimal solution.}
    \label{fig:local_search}
\end{figure}

Figure~\ref{fig:local_search} shows how local search works. In the figure, each point inside the rectangle represents a feasible solution to the problem. Local search starts from a feasible solution and evaluates its objective function. Then, it evaluates the objective function of all the solutions near the current solution and then moves to the best solution within that region. This step is represented by a violet arrow in the figure. The process then repeats until a locally optimal solution is reached, where no neighbouring solution is better than the current solution. 

Unfortunately, local search can converge to locally optimal solutions that might be far from a globally optimal solution. To deal with this issue, one option is to \emph{restart} local search when it finds a locally optimal solution and start over from a different initial solution. Another option is to use \emph{tabu search} \citep{glover1998tabu}. Tabu search is a local search approach that saves the last $n$ solutions in a \emph{tabu list} and always moves to the best neighbour that is not in the list. This allows tabu search to escape locally optimal solutions. Here, we explore these two options for solving \ref{lrm:problem}.

\begin{algorithm}[t]
    \caption{A local search approach with restarts to solve \ref{lrm:problem}}
    \label{alg:local_search}
    \input{src/local_search}
\end{algorithm}

Algorithm~\ref{alg:local_search} shows a local search approach with restarts to learn an RM. This algorithm receives the set of high-level observations $\Sigma$, the training traces $\mathcal{T}$, the maximum number of RM state $u_{\max}$, and some termination criteria such as a time limit or a maximum number of steps $t_{\max}$. The algorithm starts from a randomly generated RM (line 5). The initial RM is sampled from a uniform distribution. That is, every RM with at most $u_{\max}$ states can be selected with equal probability. On each iteration, the algorithm evaluates all neighbouring RMs (lines 9-17). We define the neighbourhood of an RM as the set of RMs that differ by exactly one transition (i.e., removing/adding a transition, or changing its value) and evaluate RMs using the objective function of \ref{lrm:problem}. When all neighbouring RMs are evaluated, the algorithm moves to the neighbouring RM with the lowest objective value (line 15), and the process repeats. If at any point a locally optimal solution is reached, then the algorithm starts over from another randomly generated RM (line 9). Finally, the best RM seen so far is returned when the terminal condition is met (line 18).

Algorithm~\ref{alg:tabu_search} shows a tabu search approach to learn a reward machine. It has the same inputs as local search, but it also receives the size of the tabu list $\tau_{\text{size}}$. Our tabu search method is identical to Algorithm~\ref{alg:local_search} except that tabu search initializes the tabu list in line 2, adds the current RM to the tabu list in line 11, and moves to the best solution that is not in the tabu list in lines 12-18.

We note that, technically, these two methods will eventually find an optimal solution. The reason is that both methods restart the search when they get stuck: Local search restarts when it reaches a locally optimal solution and tabu search restarts when it reaches a neighbourhood where all the RMs are in the tabu list. Thus, these methods either find an optimal solution during the search or restart the search.\footnote{In the case of tabu search, we are assuming that the tabu list is large enough.} Since every RM can be sampled with equal probability when restarting, both methods will eventually sample an optimal solution (or find one). That said, the space of possible RMs is so vast that we cannot expect our implementations of local search and tabu search to necessarily find optimal solutions in practice.

\begin{algorithm}[t]
    \caption{A tabu search approach to solve \ref{lrm:problem}}
    \label{alg:tabu_search}
    \input{src/tabu}
\end{algorithm}

%% file: src/generate_PT.tex
\begin{algorithmic}[1]%
\STATE {\bfseries Input:} $\{\tau_0, \ldots, \tau_n\}$
\STATE root\_node $ \leftarrow $ create\_root\_node()
\FOR{$i = 0$ \TO $n$} 
    \STATE node $ \leftarrow $ root\_node
    \FOR{$t = 0$ \TO $|\tau_i|-1$}
        \STATE node.increase\_trace\_counter\_by\_one()
        \IF {\textbf{not} node.has\_child($\sigma_{i,t}$)}
            \STATE node.add\_child($\sigma_{i,t}$)
        \ENDIF
        \STATE node $ \leftarrow $ node.get\_child($\sigma_{i,t}$)
    \ENDFOR
\ENDFOR
\RETURN root\_node
\end{algorithmic}

%% file: src/milp.tex
{%
\begingroup%
\allowdisplaybreaks
\begin{align}
\min\; & \sum_{n \in S} w_n \cdot z_n \tag{\texttt{MILP}} \label{mip:problem}\\
s.t.\; 
& z_n \geq \sum_{m = 1}^{K}y_{u,\sigma,m} \log{(m)} - (1 - x_{n,u} ) \log{(K)} & \forall n \in S, u \in U, \sigma = o(n) \label{mip:cost} \\
& \sum_{m = 1}^{K}y_{u,\sigma,m} = 1 & \forall u \in U, \sigma \in \Sigma \label{mip:oneCount}\\
& \sum_{\sigma' \in \Sigma } p_{u, \sigma, \sigma'} = \sum_{m = 1}^{K}y_{u,\sigma,m}\cdot m  & \forall u \in U, \sigma \in \Sigma \label{mip:predictionCount}\\
& p_{u, o(n), o(n')} \geq x_{n,u} & \forall u \in U, n \in S_{\text{in}}, n' \in C(n) \label{mip:prediction} \\
& \sum_{u'\in U}d_{u,\sigma,u'} = 1 & \forall u \in U, \sigma \in \Sigma \label{mip:oneJump} \\
& \sum_{u \in U} x_{n,u}  = 1 & \forall n \in S, u \in U  \label{mip:oneState}  \\
& x_{n,u_0} = 1 & n = n_{\text{root}} \label{mip:firstSate} \\
& x_{p(n),u} + x_{n,u'} - 1 \leq d_{u, o(n), u'} & \forall u, u' \in U, n \in S \label{mip:stateTransition} \\
& d_{u,\sigma,u'} \leq d_{u', \sigma, u'} & \forall u, u' \in U, u \neq u', \sigma=\Sigma \label{mip:optionalInOut} \\
& x_{n,u} \in \{0,1\} &\forall n \in S \cup \{n_{\text{root}}\}, u \in U \label{mip:domx}\\
& d_{u,\sigma,u'} \in \{0,1\} &\forall u, u' \in U, \sigma \in \Sigma \label{mip:domd} \\
& p_{u,\sigma,\sigma'} \in \{0,1\} &\forall u \in U, \sigma,\sigma' \in \Sigma \label{mip:domp} \\
& y_{u,\sigma,m} \in \{0,1\} &\forall u \in U, \sigma \in \Sigma, m \in \{1..K\} \label{mip:domy}\\
& z_{n} \geq 0 & \forall n \in S \label{mip:domz}
\end{align}
\endgroup%
}%

%% file: src/local_search.tex
\begin{algorithmic}[1]
    \STATE {\bfseries Input:} $\Sigma$, $\mathcal{T}$, $u_\text{max}$, $t_{\max}$
    \STATE $t \leftarrow 0$, $c^* \leftarrow \infty$, $\mathcal{R}^* \leftarrow$ None
    \WHILE{$t \leq t_{\max}$ }
        \STATE $c_{p} \leftarrow \infty$
        \STATE $\mathcal{R} \leftarrow$ sample\_a\_reward\_machine($\mathcal{T}$, $\Sigma$, $u_\text{max}$)
        \STATE $c \leftarrow$ evaluate\_reward\_machine($\mathcal{R}$, $\mathcal{T}$)
        \IF{$c<c^*$}
            \STATE $c^* \leftarrow c$, $\mathcal{R}^* \leftarrow \mathcal{R}$
        \ENDIF
        \WHILE{$t \leq t_{\max}$ \AND $c < c_{p}$}
            \STATE $t \leftarrow t + 1$, $c_{p} \leftarrow c$
            \STATE $\mathcal{N} \leftarrow $ get\_neighbours($\mathcal{R}$, $\Sigma$, $u_\text{max}$)
            \FOR{$\mathcal{R}_n \in \mathcal{N}$}
                \STATE $c_n \leftarrow$ evaluate\_reward\_machine($\mathcal{R}_n$, $\mathcal{T}$)
                \IF{$c_n < c$}
                    \STATE $c \leftarrow c_n$, $\mathcal{R} \leftarrow \mathcal{R}_n$
                \ENDIF
            \ENDFOR
            \IF{$c<c^*$}
                \STATE $c^* \leftarrow c$, $\mathcal{R}^* \leftarrow \mathcal{R}$
            \ENDIF
        \ENDWHILE
    \ENDWHILE
    \RETURN $\mathcal{R}^*$
\end{algorithmic}

%% file: src/tabu.tex
\begin{algorithmic}[1]
    \STATE {\bfseries Input:} $\Sigma$, $\mathcal{T}$, $u_\text{max}$, $t_{\max}$, $\tau_\text{size}$
    \STATE $\tau \leftarrow$ initialize\_tabu\_list($\tau_\text{size}$)
    \STATE $t \leftarrow 0$, $c^* \leftarrow \infty$, $\mathcal{R}^* \leftarrow$ None
    \WHILE{$t \leq t_{\max}$ }
        \STATE $\mathcal{R} \leftarrow$ sample\_a\_reward\_machine($\mathcal{T}$, $\Sigma$, $u_\text{max}$)
        \STATE $c \leftarrow$ evaluate\_reward\_machine($\mathcal{R}$, $\mathcal{T}$)
        \IF{$c<c^*$}
            \STATE $c^* \leftarrow c$, $\mathcal{R}^* \leftarrow \mathcal{R}$
        \ENDIF
        \WHILE{$t \leq t_{\max}$ \AND $\mathcal{R} \not\in \tau$}
            \STATE $t \leftarrow t + 1$, $c \leftarrow \infty$
            \STATE $\tau \leftarrow$ add\_reward\_machine\_to\_tabu\_list($\tau$, $\mathcal{R}$)
            \STATE $\mathcal{N} \leftarrow $ get\_neighbours($\mathcal{R}$, $\Sigma$, $u_\text{max}$)
            \FOR{$\mathcal{R}_n \in \mathcal{N} \setminus \tau$}
                \STATE $c_n \leftarrow$ evaluate\_reward\_machine($\mathcal{R}_n$, $\mathcal{T}$)
                \IF{$c_n < c$}
                    \STATE $c \leftarrow c_n$, $\mathcal{R} \leftarrow \mathcal{R}_n$
                \ENDIF
            \ENDFOR
            \IF{$c<c^*$}
                \STATE $c^* \leftarrow c$, $\mathcal{R}^* \leftarrow \mathcal{R}$
            \ENDIF
        \ENDWHILE
    \ENDWHILE
    \RETURN $\mathcal{R}^*$
\end{algorithmic}

%% file: sections/06-policy.tex
\section{Simultaneously Learning a Reward Machine and a Policy}
\label{sec:main_loop}

We now describe our overall approach to simultaneously finding an RM and exploiting that RM to learn a policy. Algorithm~\ref{alg:main_loop} shows the complete pseudo-code. Our approach starts by collecting a training set of traces $\mathcal{T}$ generated by following a random policy during $t_\text{w}$ ``warmup'' steps (line 2). This set of traces is used to find an initial RM $\mathcal{R}$ using one of our discrete optimization models (line 3). The algorithm then sets the RM state to $u_0$, sets the current high-level observation $\sigma$ to $L(\emptyset,\emptyset,o)$, and initializes the policy $\pi$ (lines 4-5). The standard RL loop is then followed (lines 6-19): an action $a$ is selected according to $\pi(a|o,u)$ and the agent receives the next observation $o'$ and the immediate reward $r$. The RM state is then updated to $u' = \delta_u(u,L(o,a,o'))$ and the last experience $(o, u, a, r, o', u')$ is used to update $\pi$. Finally, the environment and RM are reset if a terminal state is reached (lines 17-18).

\begin{algorithm}[t]
    \caption{Algorithm to simultaneously learn a reward machine and a policy}
    \label{alg:main_loop}
    \input{src/main_loop.tex}
\end{algorithm}

If on any step, there is evidence that the current RM might not be perfect, our approach will attempt to find a new one (lines 11-16). Recall that the RM $\mathcal{R}$ was selected using the cardinality of its prediction sets $N$, where $N_{u,\sigma}$ is the set of high-level observations seen from the RM state $u$ immediately after observing $\sigma$ in the training data. If the current high-level observation $\sigma'$ is not in $N_{u,\sigma}$, then adding the current trace to $\mathcal{T}$ will increase the size of $N_{u,\sigma}$ for $\mathcal{R}$ and, in consequence, $\mathcal{R}$ may no longer be the best RM. Therefore, if $\sigma' \not\in N_{u,\sigma}$, we add the current trace to $\mathcal{T}$ and learn a new RM. Our method only uses the new RM if its cost is lower than $\mathcal{R}$'s and, if the RM is updated, a new policy is learned from scratch (lines 14-16).

Given the current RM, we can use any RL algorithm to learn a policy $\pi(a|o,u)$, by treating the combination of $o$ and $u$ as the current state. If the RM is perfect, then the optimal policy $\pi^*(a|o,u)$ will also be optimal for the original POMDP (as stated in Theorem~\ref{theo:rm_policy}). In this case, we can ignore the reward $\delta_r$ that comes from the RM and only consider the reward received directly from the environment. However, to further exploit the problem structure exposed by the RM (such as with QRM), we need to set $\delta_r$. We do so using the empirical average, as described in Section~\ref{sec:learn_rms}.

Let us now explain how we incorporate QRM into this process. As explained in Section~\ref{sec:RMsPO}, standard QRM under partial observability can introduce a bias because an experience $e=(o,a,o')$ might be more or less likely depending on the RM state that the agent was in when the experience was collected. We partially address this issue by updating $Q_u$ using $(o,a,o')$ iff $L(o,a,o') \in N_{u,\sigma}$, where $\sigma$ was the current high-level observation that generated the experience $(o,a,o')$. Hence, we do not transfer experiences from $u_i$ to $u_j$ if the current RM does not believe that $(o,a,o')$ is possible in $u_j$. For example, consider the cookie domain and the perfect RM from Figure~\ref{fig:rm_perfect}. If some experience consists of entering to the green room and seeing a cookie, then this experience will not be used by states $u_0$ and $u_3$ as it is impossible to observe a cookie at the green room from those states. While adding this rule works in many cases, it does not fully address the problem. We further discuss this issue in Section~\ref{sec:limitations}.

%% file: src/main_loop.tex
\begin{algorithmic}[1]
    \STATE {\bfseries Input:} $\mathcal{P}$, $L$, $u_\text{max}$, $t_\text{w}$
    \STATE $\mathcal{T} \leftarrow $ collect\_traces($t_\text{w}$)
    \STATE $\mathcal{R}, N \leftarrow $ learn\_rm($\mathcal{P}$, $L$, $\mathcal{T}$, $u_\text{max}$)
    \STATE $o \leftarrow $ env\_get\_initial\_state(), $u \leftarrow u_0$, $\sigma \leftarrow L(\emptyset,\emptyset,o)$
    \STATE $\pi \leftarrow $ initialize\_policy()
    \FOR{$t = 1$ \TO $t_\text{train}$} 
        \STATE $a \leftarrow $ select\_action($\pi$, $o$, $u$)
        \STATE $o', r, \text{done} \leftarrow $ env\_execute\_action($a$)
        \STATE $u', \sigma' \leftarrow \delta_u(u,L(o,a,o')), L(o,a,o')$
        \STATE $\pi \leftarrow $ improve($\pi$, $o$, $u$, $\sigma$, $a$, $r$, $o'$, $u'$, $\sigma'$, done, $N$)
        \IF{$\sigma' \not\in N_{u,\sigma}$} 
            \STATE $\mathcal{T} \leftarrow \mathcal{T} \cup{}$get\_current\_trace()
            \STATE $\mathcal{R}', N \leftarrow $ relearn\_rm($\mathcal{R}$, $\mathcal{P}$, $L$, $\mathcal{T}$, $u_\text{max}$)
            \IF{$\mathcal{R} \neq \mathcal{R}'$} 
                \STATE $\text{done} \leftarrow $ \TRUE, $\mathcal{R} \leftarrow \mathcal{R}'$
                \STATE $\pi \leftarrow $ initialize\_policy()
            \ENDIF
        \ENDIF
        \IF{done} 
            \STATE $o' \leftarrow $ env\_get\_initial\_state(), $u' \leftarrow u_0$, $\sigma' \leftarrow L(\emptyset,\emptyset,o)$
        \ENDIF
        \STATE $o \leftarrow o'$, $u \leftarrow u'$, $\sigma \leftarrow \sigma'$
    \ENDFOR
    \RETURN $\pi$
\end{algorithmic}

%% file: sections/07-experiments.tex
\section{Experimental Evaluation}
\label{sec:results}

In this section, we provide an empirical evaluation of our method in three partially observable environments. Our evaluation consists of two parts. First, we compare the effectiveness of our mixed integer linear programming model (MILP), our constrained programming model (CP), our local search with restart algorithm (LS), and our tabu search method (TS) to solve \ref{lrm:problem}. We then show how the combination of learning an RM and a policy using double DQN (LRM+DDQN) and deep QRM (LRM+DQRM) compares to different baselines. As a brief summary, our results show the following:
\begin{enumerate}
    \itemsep0em 
    \item MILP and CP find optimal solutions for small instances of \ref{lrm:problem}.
    \item LS and TS find better solutions than MILP and CP for large instances of \ref{lrm:problem}.
    \item LS consistently finds better solutions than TS.
    \item Our LRM-based methods can outperform A3C, ACER, PPO, and DDQN.
    \item LRM+DQRM learns faster then LRM+DDQN, but it is less stable.
\end{enumerate}

\subsection{Domains}
\label{sec:domains}

We tested our approach on three partially observable domains, shown in Figure~\ref{fig:domains2}. These environments consist of three rooms connected by a hallway. The agent can move in the four cardinal directions but its actions fail with a 5\% probability. The agent can only see what it is in the room that it currently occupies, as shown in Figure~\ref{fig:d_view2}. What makes these tasks difficult is the hallway. The hallway forces the agent to observe long sequences of identical observations multiple times to solve a task. However, depending on previous observations, the optimal actions and expected returns will be completely different when the agent is in the hallway.

The first environment is the \emph{cookie domain} (Figure~\ref{fig:d_cookie2}) described in Section~\ref{sec:RMsPO}. Each episode is $5,000$ steps long, during which the agent should attempt to get as many cookies as possible. To do so, it has to press the button in the orange room and then look for the cookie that is delivered to the blue or green room.

The second environment is the \emph{symbol domain} (Figure~\ref{fig:d_symbol}). This domain has three symbols $\clubsuit$, $\spadesuit$, and $\blacklozenge$ in the blue and green rooms. At the beginning of an episode, one symbol from $\{\clubsuit, \spadesuit, \blacklozenge\}$ and possibly a right or left arrow are randomly placed at the orange room. Intuitively, that symbol and arrow will tell the agent where to go, for example,  $\clubsuit$ and $\rightarrow$ tell the agent to go to $\clubsuit$ in the east room. If there is no arrow, the agent can go to the target symbol in either room. An episode ends when the agent reaches any symbol in the blue or green room, at which point it receives a reward of $+1$ if it reached the correct symbol and $-1$ otherwise. All other steps in the environment provide no reward.

The third environment is the \emph{2-keys domain} (Figure~\ref{fig:d_key}). The agent receives a reward of $+1$ when it reaches the coffee in the orange room. To do so, it must open the two doors, shown in brown. Each door requires a different key to open it, and the agent can only carry one key at a time. At the beginning of each episode, the two keys are randomly located in either the green room, the blue room, or split between them. To solve this problem, the agent must keep track of the locations of the keys.

\begin{figure}
    \centering
    \begin{subfigure}[b]{.49\textwidth}
        \centering
        \includegraphics[width=0.75\textwidth]{figs/room-po.pdf}
        \subcaption{Agent's view.}
        \label{fig:d_view2}
    \end{subfigure}
    \begin{subfigure}[b]{.49\textwidth}
        \centering
        \includegraphics[width=0.75\textwidth]{figs/cookie.pdf}
        \subcaption{Cookie domain.}
        \label{fig:d_cookie2}
    \end{subfigure}
    
    \vspace{2mm}
    \begin{subfigure}[b]{.49\textwidth}
        \centering
        \includegraphics[width=0.75\textwidth]{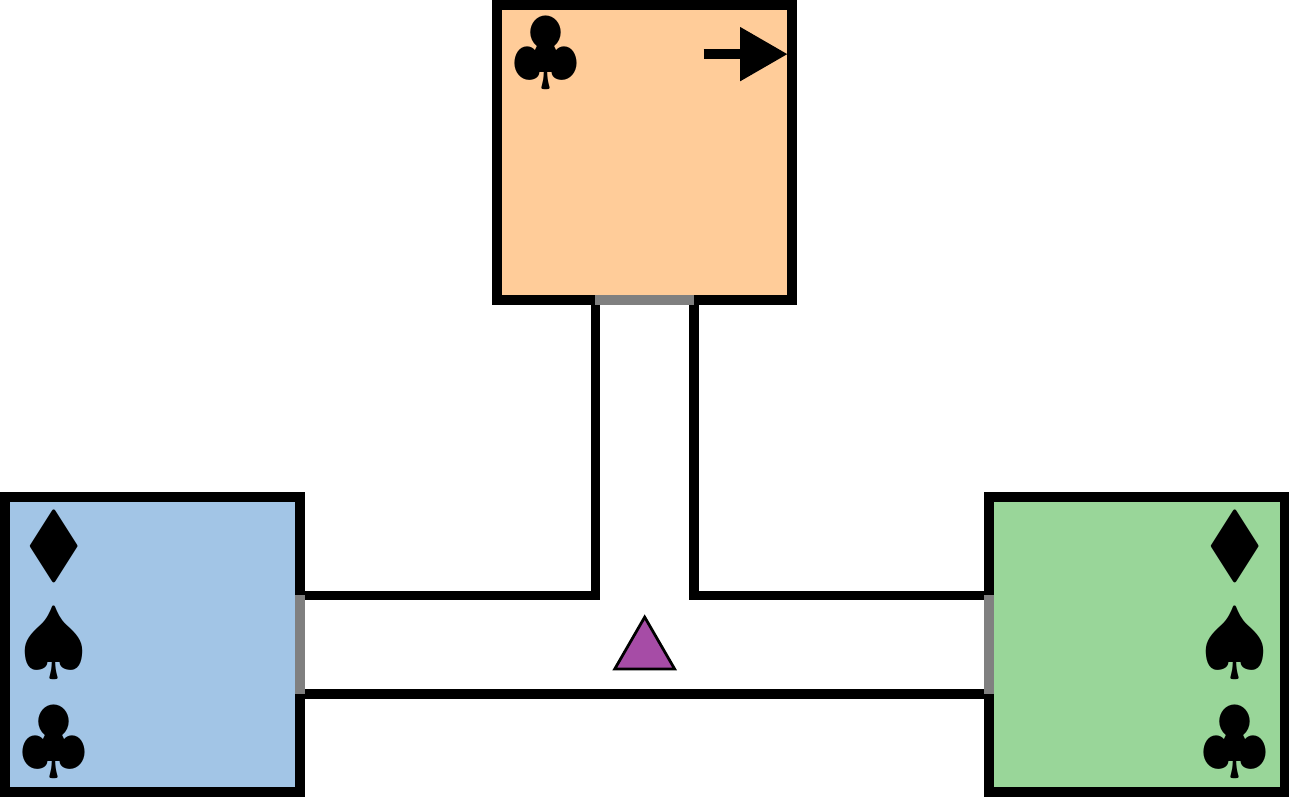}
        \subcaption{Symbol domain.}
        \label{fig:d_symbol}
    \end{subfigure}
    \begin{subfigure}[b]{.49\textwidth}
        \centering
        \includegraphics[width=0.75\textwidth]{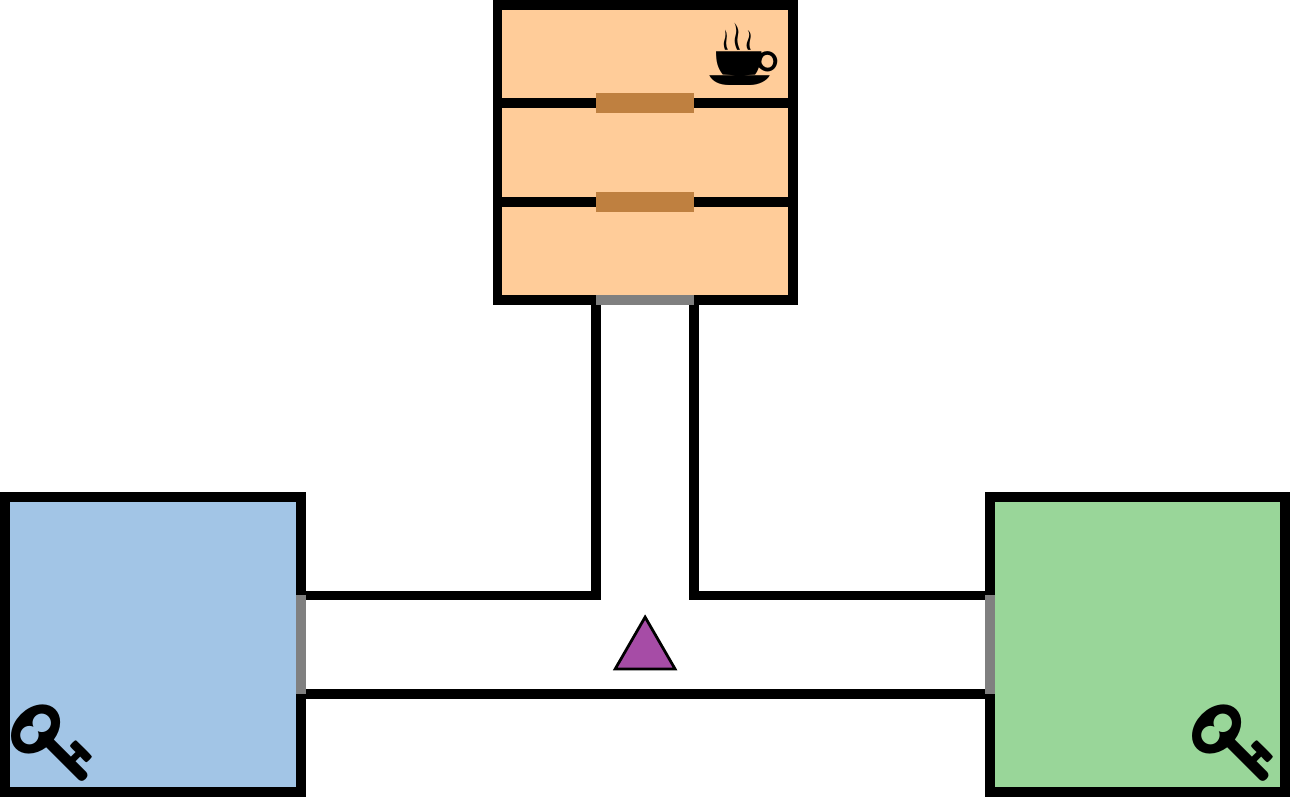}
        \subcaption{2-keys domain.}
        \label{fig:d_key}
    \end{subfigure}
    \caption{Partially observable domains where the agent
can only see what is in the current room.}
    \label{fig:domains2}
\end{figure}

\subsection{Comparisons Between the Discrete Optimization Models}

We first compare the performance of our four models for solving different instances of \ref{lrm:problem}. The objective of this experiment is to compare how well each model scales as we increase the size of the training data and the size of the reward machine. To that end, we generated 20 training sets per domain, where each training set consists of $10^3$, $10^4$, $10^5$, or $10^6$ experiences collected by following a uniformly random policy. We sampled five training sets per each possible size and learned reward machines with at most $5$ or $10$ states. This gave a total of $120$ problems instances of \ref{lrm:problem}.

\begin{table}[t]
    \centering
    \caption{Comparing different models for solving \ref{lrm:problem} in problem instances with training sets varying from $10^3$ to $10^6$ experiences and $u_{\max} \in \{5, 10\}$. Each row includes five problem instances.}
    \label{tab:LRM_model_comparison}
    {\small  \setlength{\tabcolsep}{4.9pt}
        \begin{tabular}{ crrrrrrrrrrr }
        \toprule
        \multicolumn{3}{c}{\textbf{Configuration}} & \multicolumn{4}{c}{\textbf{Avg. objective}} & \multicolumn{4}{c}{\textbf{No. best}} \\
        \cmidrule(rl){1-3} \cmidrule(rl){4-7}  \cmidrule(rl){8-11}
        Dataset & $|\mathcal{T}|$  & $|\mathcal{T}_c|$ & MILP & CP & LS & TS & MILP & CP & LS & TS \\
        \midrule
          & $10^3$ & 59 & \textbf{12.6} & \textbf{12.6} & 14.2 & 13.8 & \textbf{5} & \textbf{5} & 1 & 1\\
        Cookie& $10^4$ & 487 & 237.3 & \textbf{226.7} & 229.1 & 230.2 & 1 & \textbf{5} & 2 & 0\\
        $(u_{\max}=5)$& $10^5$ & 4943 & 3097.0 & 2700.4 & \textbf{2699.7} & 2719.7 & 0 & \textbf{3} & 2 & 0\\
        & $10^6$ & 48663 & 31075.1 & 28226.1 & \textbf{26462.6} & 26833.3 & 0 & 0 & \textbf{5} & 0\\
        \midrule
        & $10^3$ & 59 & \textbf{6.5} & 6.8 & 9.6 & 10.5 & \textbf{5} & 4 & 0 & 0\\
        Cookie& $10^4$ & 487 & 233.3 & 204.6 & 206.0 & \textbf{204.5} & 0 & 1 & 0 & \textbf{4}\\
        $(u_{\max}=10)$& $10^5$ & 4943 & 3197.0 & 2713.7 & \textbf{2658.9} & 2696.8 & 0 & 0 & \textbf{5} & 0\\
        & $10^6$ & 48663 & 30709.5 & 28366.8 & \textbf{26461.7} & 27092.0 & 0 & 0 & \textbf{5} & 0\\
        \midrule
        & $10^3$ & 41 & \textbf{21.0} & \textbf{21.0} & \textbf{21.0} & \textbf{21.0} & \textbf{5} & \textbf{5} & \textbf{5} & \textbf{5}\\
        Symbol& $10^4$ & 268 & \textbf{218.8} & \textbf{218.8} & \textbf{218.8} & 220.0 & \textbf{5} & \textbf{5} & \textbf{5} & 3\\
        $(u_{\max}=5)$& $10^5$ & 2597 & 3423.7 & 2897.9 & \textbf{2896.7} & 2902.3 & 0 & 3 & \textbf{5} & 2\\
        & $10^6$ & 25875 & 36705.4 & 29689.9 & \textbf{29687.7} & 29688.8 & 0 & 4 & \textbf{5} & 3\\
        \midrule
        & $10^3$ & 41 & \textbf{16.2} & \textbf{16.2} & 16.5 & 16.4 & \textbf{5} & \textbf{5} & 2 & 2\\
        Symbol& $10^4$ & 268 & 185.8 & \textbf{181.2} & 181.5 & 185.0 & 1 & \textbf{5} & 3 & 0\\
        $(u_{\max}=10)$& $10^5$ & 2597 & 3416.1 & 2620.7 & \textbf{2583.5} & 2620.4 & 0 & 0 & \textbf{5} & 0\\
        & $10^6$ & 25875 & 36216.6 & 27992.9 & \textbf{27050.4} & \textbf{27050.4} & 0 & 0 & \textbf{5} & \textbf{5}\\
        \midrule
        & $10^3$ & 42 & \textbf{6.9} & \textbf{6.9} & 7.6 & 7.5 & \textbf{5} & \textbf{5} & 2 & 2\\
        2-Keys& $10^4$ & 378 & 196.5 & \textbf{176.7} & 176.9 & 180.7 & 1 & \textbf{5} & 3 & 1\\
        $(u_{\max}=5)$& $10^5$ & 3690 & 3713.2 & 2364.7 & \textbf{2349.6} & 2391.4 & 0 & 0 & \textbf{5} & 0\\
        & $10^6$ & 37923 & 38875.3 & 29379.9 & \textbf{24397.0} & 24762.1 & 0 & 0 & \textbf{5} & 0\\
        \midrule
        & $10^3$ & 42 & \textbf{3.5} & \textbf{3.5} & 5.4 & 5.1 & \textbf{5} & \textbf{5} & 0 & 0\\
        2-Keys& $10^4$ & 378 & 184.4 & 151.6 & \textbf{145.4} & 157.2 & 0 & 0 & \textbf{5} & 0\\
        $(u_{\max}=10)$& $10^5$ & 3690 & 3746.1 & 2363.8 & \textbf{2210.6} & 2237.6 & 0 & 0 & \textbf{5} & 0\\
        & $10^6$ & 37923 & 38087.0 & 29065.0 & \textbf{23352.9} & 23558.8 & 0 & 0 & \textbf{5} & 0\\
        \midrule
        \multicolumn{3}{c}{\textbf{Average/Total}} & 9732.7 & 7900.3 & \textbf{7251.8} & 7325.2 & 38 & 60 & \textbf{85} & 28\\
        \bottomrule
    \end{tabular}}
\end{table}

Each approach was run with a 10-minute time limit using 62 cores on a Threadripper 2990WX processor with 124GB of RAM. We used Gurobi 9.1 \citep{gurobi} to solve the MILP model and IBM ILOG CP Optimizer 12.8 \citep{CPOManual} for the CP model. These are sophisticated state-of-the-art solvers. In contrast, we used a simple Python implementation of local search and tabu search in our experiments. We set $\tau_\text{size}=100$ for tabu search. We note that local search and tabu search are stochastic approaches that, in contrast to MILP and CP, might find a different solution on each run. For that reason, we ran local search and tabu search 5 times per problem instance and report the average cost across those runs.

Table~\ref{tab:LRM_model_comparison} shows the final results. Each row shows the aggregated results over five problem instances that share the same domain (i.e., cookie, symbol, or 2-keys), maximum number of RM states (i.e., $u_{\max} \in \{5,10\}$), and size of the training set (i.e., $|\mathcal{T}| \in \{10^3,10^4,10^5,10^6\}$). Each row also shows the average size of the training set $|\mathcal{T}_c|$ after the traces are compressed (as described in Section~\ref{sec:lrm_models}). The table reports the average objective function of each model, where lower is better, and the number of instances where each model found the best solution among all others.

For training sets with less than $10,000$ experiences, our CP model tends to find the best solutions. However, for larger instance, local search methods dominate. Note that continuously restarting local search is a better strategy for learning RMs than using a tabu list in these domains. Still, the performance of TS is not too far from LS and, hence, we test both approaches for learning RMs in our next experiments.

\subsection{Reinforcement Learning Experiments}
\label{sec:rl_results}

We tested two versions of our \emph{learned reward machine (LRM)} method: LRM+DDQN and LRM+DQRM. Both learn RMs from experience but LRM+DDQN learns a policy using DDQN \citep{van2016deep} while LRM+DQRM uses the modified version of QRM described in Section~\ref{sec:main_loop}. To learn the reward machine, these approaches solve \ref{lrm:problem} using local search with restarts or tabu search. In all domains, we used $u_\text{max} = 10$, $t_\text{max} = 100$, $\tau_\text{size} = 100$, $t_\text{w} = 200,000$, an epsilon greedy policy with $\epsilon = 0.1$, and a discount factor $\gamma = 0.9$. We compared against 4 baselines: DDQN \citep{van2016deep}, A3C \citep{mnih2016asynchronous}, ACER \citep{wang2016sample}, and PPO \citep{schulman2017proximal}. 
DDQN uses the concatenation of the last 10 observations as input which gives DDQN a limited memory to better handle the domains. A3C, ACER, and PPO use an LSTM to summarize the history. Note that the output of the labelling function was also given to the baselines, as described below. 

\subsubsection{Hyperparameters and Features}
For LRM+DDQN and LRM+DQRM, the neural network used has 5 fully connected layers with 64 neurons per layer. On every step, we trained the network using 32 sampled experiences from a replay buffer of size 100,000 and a learning rate of $5\cdot 10^{-5}$. The target networks were updated every $100$ steps. 

The DDQN baseline uses the same parameters and network architecture as our LRM methods, but its input is the concatenation of the last 10 observations, as commonly done by Atari playing agents \citep{mnih2015human}. This gives DDQN a limited memory to better handle partially observable domains. We note that since the optimal path from any one room to another is less than 10 steps, giving the agent the last 10 observations means that the agent has enough information to perfectly summarize its history if it can figure out how to do so. The rest of the baselines, namely A3C, ACER, and PPO, use an LSTM to summarize the history. 

To select hyperparameters for A3C, ACER, and PPO, we followed the same methodology that was used in their original publications. We ran each approach at least 30 times per domain, and on every run, we randomly selected the number of hidden neurons for the LSTM from $\{64, 128, 256, 512\}$ and a learning rate from (1e-3, 1e-5). We also sampled $\delta$ from $\{0,1,2\}$ for ACER and the clip range from $(0.1,0.3)$ for PPO. Other parameters were fixed to their default values.

While interacting with the environment, the agents were given a ``top-down" view of the world represented as a set of binary matrices. One matrix had a 1 in the current location of the agent, one had a 1 in only those locations that are currently observable, and the remaining matrices each corresponded to an object in the environment and had a 1 at only those locations that were both currently observable and contained that object (i.e., locations in other rooms are ``blacked out"). The agent also had access to features indicating if they were carrying a key, which color room they were in, and the current status of the events detected by the labelling function.

\begin{figure}[t]
    \centering
    \begin{subfigure}{0.49\columnwidth}
        \centering
        \includegraphics[]{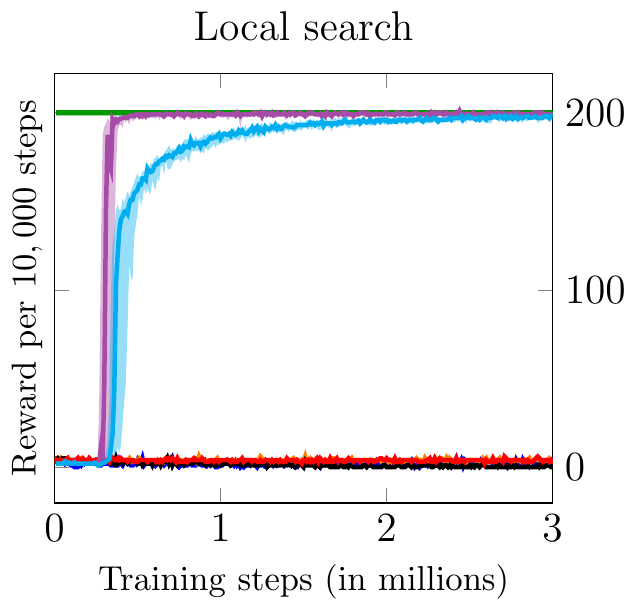}
    \end{subfigure}
    \begin{subfigure}{0.49\columnwidth}
        \centering
        \includegraphics[]{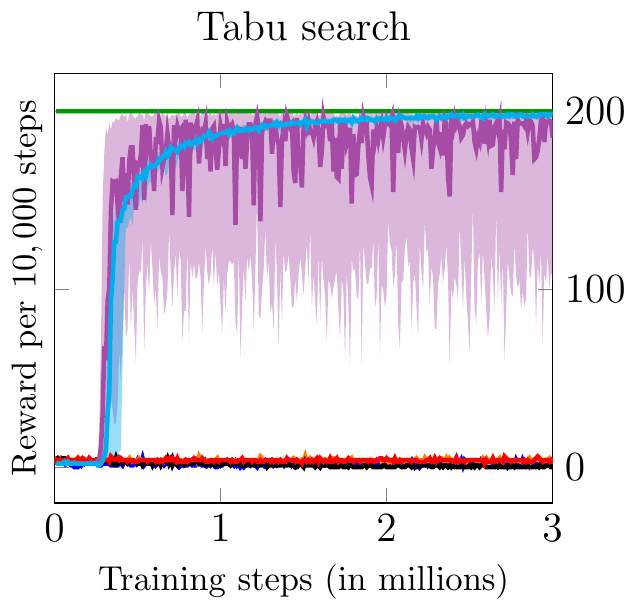}
    \end{subfigure}
    
    \vspace{1mm}
    \input{results/legend.tex}
    \caption{Results on the cookie domain. LRM is solved using local search or tabu search.}
    \label{fig:cookie-results}
\end{figure}

\subsubsection{Results}

\begin{figure}[t]
    \centering
    \begin{subfigure}{0.49\columnwidth}
        \centering
        \includegraphics[]{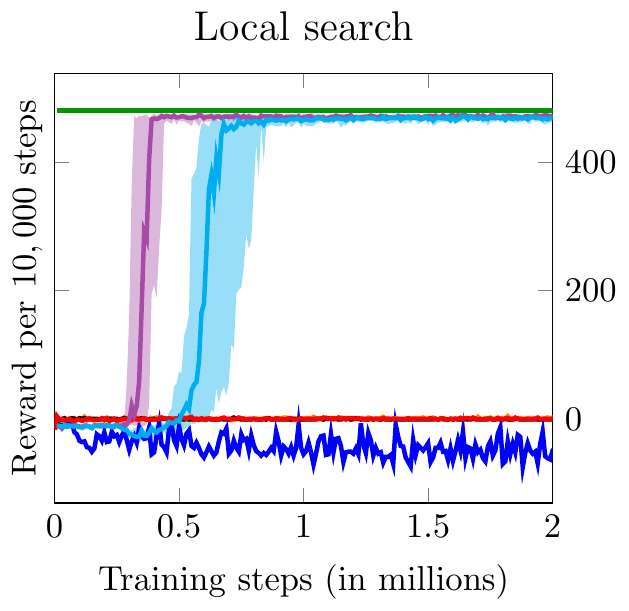}
    \end{subfigure}
    \begin{subfigure}{0.49\columnwidth}
        \centering
        \includegraphics[]{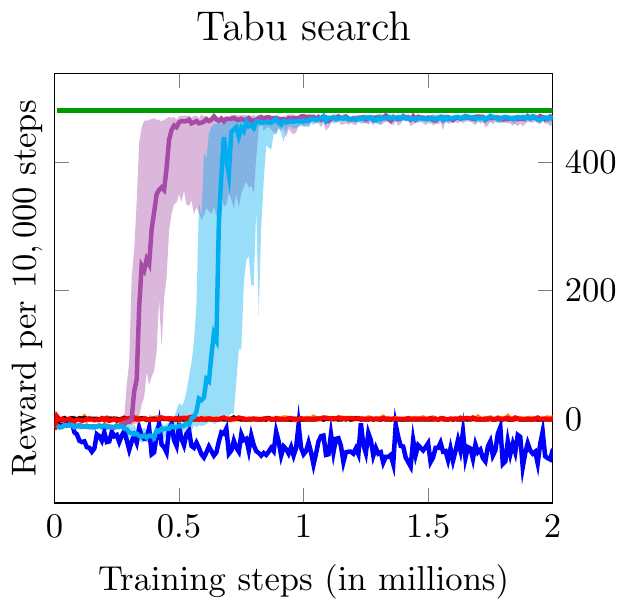}
    \end{subfigure}
    
    \vspace{1mm}
    \input{results/legend.tex}
    \caption{Results on the symbol domain. LRM is solved using local search or tabu search.}
    \label{fig:symbol-results}
\end{figure}

\begin{figure}[t]
    \centering
    \begin{subfigure}{0.49\columnwidth}
        \centering
        \includegraphics[]{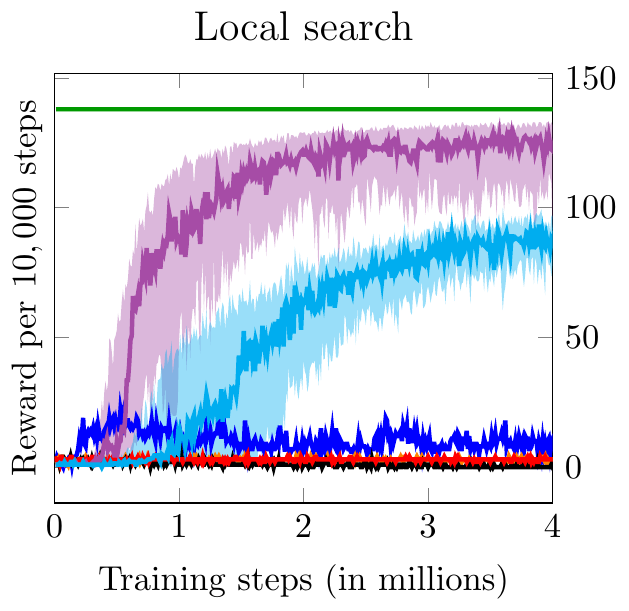}
    \end{subfigure}
    \begin{subfigure}{0.49\columnwidth}
        \centering
        \includegraphics[]{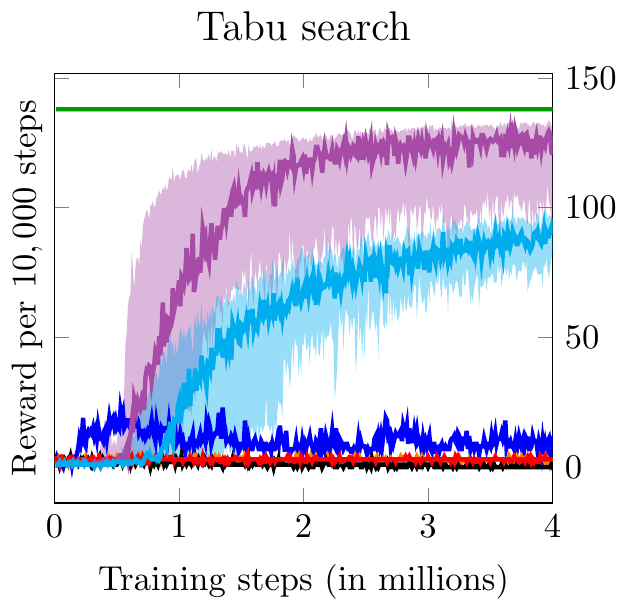}
    \end{subfigure}
    
    \vspace{1mm}
    \input{results/legend.tex}
    \caption{Results on the 2-keys domain. LRM is solved using local search or tabu search.}
    \label{fig:keys-results}
\end{figure}

Figures \ref{fig:cookie-results}, \ref{fig:symbol-results}, and \ref{fig:keys-results} show the total cumulative rewards that each approach gets every $10,000$ training steps and compares it to the optimal policy. For the LRM algorithms, the figures show the median performance over 30 runs per domain, and percentile 25 to 75 in the shadowed area. For the DDQN baseline, we show the maximum performance seen for each time period over 5 runs per problem. Similarly, we also show the maximum performance over the 30 runs of A3C, ACER, and PPO per period. All the baselines outperformed a random policy, but none make much progress on any of the domains. Each figure shows two settings. In the left, it shows the performance when LRM is solved using local search with restarts. In the right, it shows the case where LRM is solved using tabu search. Note that this only affects the LRM methods. The baselines' performance is identical in the left and right figures.

As the results show, LRM-based methods largely outperform all the baselines in these domains, reaching an optimal policy in the cookie domain (Figure~\ref{fig:cookie-results}) and a close-to-optimal policy in the symbol domain (Figure~\ref{fig:symbol-results}). We also note that LRM+DQRM learns faster than LRM+DDQN. In particular, LRM+DQRM converged to considerably better policies in the 2-keys domain (Figure~\ref{fig:keys-results}). However, LRM+DQRM is more unstable than LRM+DDQN when solving \ref{lrm:problem} via tabu search. We believe this behaviour is due to two factors. First, tabu search is likely finding worse solutions than local search, as suggested by Table~\ref{tab:LRM_model_comparison}. Second, QRM exploits the structure of the learned RM. Thus, it is reasonable to expect that converging to a suboptimal RM would hurt the performance of DQRM more than the performance of DDQN.

%% file: results/legend.tex
\begin{tikzpicture}
\node[draw=black] {%
	\begin{tabular}{llll}
	\textbf{Legend:} 
	& \entry{blue}{DDQN} 
	& \entry{orange}{A3C} 
	& \entry{cyan}{LRM + DDQN}\\ 	
	\entry{green!60!black}{Optimal}
	& \entry{red}{ACER} 
	& \entry{black}{PPO}  
	& \entry{violet!70!white}{LRM + DQRM}
	\end{tabular}};
\end{tikzpicture}

%% file: sections/08-discussion.tex
\section{Discussion}
\label{sec:limitations}

Solving partially observable RL problems is challenging and LRM was able to solve three problems that were conceptually simple but presented a major challenge to A3C, ACER, and PPO with LSTM-based memories. A key idea behind these results was to optimize over a necessary condition for perfect RMs. This objective favors RMs that are able to predict possible and impossible future observations at the abstract level given by the labelling function $L$. In this section, we discuss the advantages and current limitations of such an approach.

We begin by considering the performance of local search methods in our domains. Given a training set composed of one million transitions, our simple Python implementation of local search takes less than 2.5 minutes to learn an RM across all our environments, when using 62 workers on a Threadripper 2990WX processor and $t_{\max}=100$. Note that local search's main bottleneck is evaluating the neighbourhood around the current RM solution. As the size of the neighbourhood depends on the size of the set of propositional symbols $\mathcal{P}$, exhaustively evaluating the neighbourhood may sometimes become impractical. To handle such problem, we might import ideas from the \emph{large neighborhood search} literature \citep{pisinger2010large}.

Regarding limitations, learning the RM at the abstract level is efficient but requires ignoring (possibly relevant) low-level information. For instance, Figure~\ref{fig:d_gravity} shows an adversarial example for LRM. The agent receives reward for eating the cookie ($\textcookieeaten$). There is an external force pulling the agent down -- i.e., the outcome of the ``\texttt{move-up}'' action is actually a downward movement with high probability. The agent can press a button ($\textcookiebutton$) to turn off (or back on) the external force. Hence, the optimal policy is to press the button and then eat the cookie. Given $\mathcal{P} = \{\textcookieeaten, \textcookiebutton\}$, a perfect RM for this environment is fairly simple (see Figure~\ref{fig:grav_rm}) but LRM might not find it, even if the traces are not compressed. The reason is that pressing the button changes the low-level probabilities in the environment but does not change what is possible or impossible at the abstract level. In other words, while the LRM objective optimizes over necessary conditions for finding a perfect RM, those conditions are not sufficient to ensure that an optimal solution will be a perfect RM. In addition, if a perfect RM is found, our heuristic approach to share experiences in QRM would not work as intended because the experiences collected when the force is on (at $u_0$) would be incorrectly used to update the policy for the case where the force is off (at $u_1$). 

\begin{figure}
    \centering
    \begin{subfigure}[b]{.35\textwidth}
        \centering
        \includegraphics[width=0.6\textwidth]{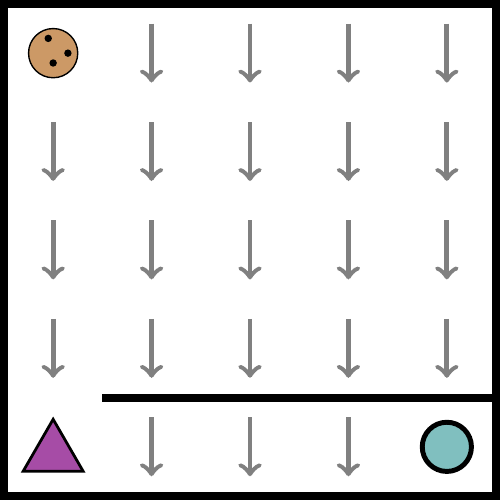}
        \subcaption{The gravity domain.}
        \label{fig:d_gravity}
    \end{subfigure}
    \begin{subfigure}[b]{.63\textwidth}
        \centering
        \includegraphics[]{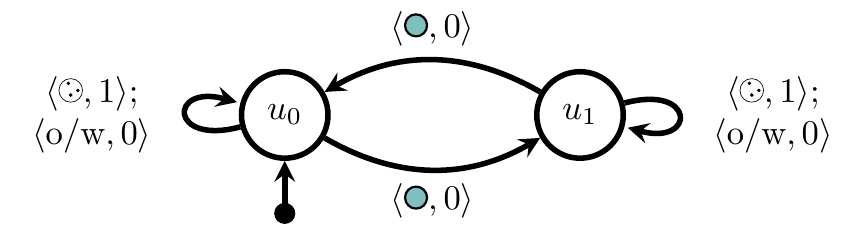}
        \vspace{1mm}
        
        \subcaption{A perfect RM for the gravity domain.}
        \label{fig:grav_rm}
    \end{subfigure}
    \caption{A partially observable environment where the agent
cannot see the external force.}
    \label{fig:domains3}
\end{figure}

Other current limitations include that it is unclear how to handle noise over the high-level detectors $L$ and how to transfer learning from previously learned policies when a new RM is learned. Finally, defining a set of proper high-level detectors for a given environment might be a challenge to deploying LRM. Hence, looking for ways to automate that step is an important direction for future work.

%% file: sections/09-related_work.tex
\section{Related Work}

State-of-the-art approaches to partially observable RL use Recurrent Neural Networks (RNNs) as memory in combination with policy gradient \citep[e.g.,][]{mnih2016asynchronous,wang2016sample,schulman2017proximal,jaderberg2016reinforcement} or use external neural-based memories \citep[e.g.,][]{oh2016control,khan2017memory,hung2018optimizing}. Other approaches include extensions to Model-Based Bayesian RL that work under partial observability \citep[e.g.,][]{poupart2008model,doshi2013bayesian,ghavamzadeh2015bayesian} or provide a small binary memory to the agent and a special set of actions to modify it \citep{peshkin1999learning}. While our experiments highlight the merits of our approach with respect to RNN-based approaches, we rely on ideas that are largely orthogonal. As such, there is significant potential in mixing these approaches to get the benefit of memory at both the high- and the low-level.

The effectiveness of automata-based memory has long been recognized in the POMDP literature \citep{cassandra1994acting}, where the objective is to find policies given a complete specification of the environment. 
The idea is to encode policies using \emph{finite state controllers (FSCs)} which are \emph{finite state machines (FSMs)} where the transitions are defined in terms of low-level observations from the environment and each state in the FSM is associated with one primitive action. When interacting with the environment, the agent always selects the action associated with the current state in the controller. \cite{meuleau1999learning} adapted this idea to work in the RL setting by exploiting policy gradient to learn policies encoded as FSCs. RMs can be considered as a generalization of FSC as they allow for transitions using conditions over high-level events and associate complete policies (instead of just one primitive action) to each state. This property allows our approach to easily leverage existing deep RL methods to learn policies from low-level inputs, such as images -- which is not achievable by \cite{meuleau1999learning}. That said, further investigating using ideas for learning FSMs \citep[e.g.,][]{angluin1983inductive,zeng1993learning,giantamidis2016learning,shvo2021interpretable} in learning RMs is a promising direction for future work.

Our approach to learn RMs is greatly influenced by \emph{predictive state representations (PSRs)} \citep{littman2002predictive}. The idea behind PSRs is to find a set of core tests (i.e., sequences of actions and observations) such that if the agent can predict the probabilities of these occurring, given any history $H$, then those probabilities can be used to compute the probability of any other test given $H$. The insight is that state representations that are good for predicting the next observation are good for solving partially observable environments. We adapted this idea to the context of RM learning.

Finally, we note that different approaches to learn RMs were proposed simultaneously, or shortly after, our original publication \citep[e.g.,][]{xu2020joint,xu2020active,furelos2020inductionAAAI,rens2020online,gaon2019reinforcement,memarian2020active,neider2021advice,hasanbeig2019deepsynth}. They all learn reward machines in fully observable domains. Their goal is to learn the smallest RM that is consistent with the reward function -- which makes sense for fully observable domains, but would have limited utility under partial observability (as discussed in Section~\ref{sec:approach}). 

Since they stay in the fully-observable setting, they can use off-the-shelf automata learning approaches to learn the RM. These include methods that learn reward machines using a SAT solver \citep{xu2020joint,neider2021advice}, use inductive logic programming \citep{furelos2020inductionAAAI}, and by using program synthesis \citep{hasanbeig2019deepsynth}. There has also been work on adapting the $L^*$ algorithm \citep{angluin1987learning} to learn RMs given the model of the MDP \citep{rens2020online}, expert demonstrations \citep{memarian2020active}, or in a pure RL setting \citep{gaon2019reinforcement,xu2020active}.

Besides proposing approaches to learn reward machines for fully-observable problems, these works also make additional contributions that may be useful in the context of partial observability. For instance, \cite{furelos2020inductionAAAI} and \cite{hasanbeig2019deepsynth} add a reward shaping procedure to encourage exploration. \cite{xu2020joint} propose a simple mechanism to transfer some of the previously learned Q-value estimates when a new reward machine is learned. \cite{neider2021advice} show how to incorporate domain knowledge when learning a reward machine. Finally, \cite{gaon2019reinforcement} and \cite{xu2020active} allow, in some cases, driving the agent's exploration towards finding \emph{bugs} in the reward machine. Further study into how to use these in the case of partial observability is left as future work.

%% file: sections/10-conclusion.tex
\section{Concluding Remarks}

We have presented a method for learning reward machines in partially observable environments and demonstrated the effectiveness of doing so to tackle partially observable RL problems that are unsolvable by the state-of-the art deep RL methods A3C, ACER and PPO. Informed by criteria from the POMDP, FSC, and PSR literature, we proposed a set of RM properties that support tackling RL in partially observable environments. We used these properties to formulate RM learning as a discrete optimization problem. We experimented with several optimization methods, finding local search methods to be the most effective. We then combined this RM learning with policy learning for solving partially observable RL problems. Our combined approach outperformed a set of strong LSTM-based approaches on different domains. 

We believe this work represents an important building block for creating RL agents that can solve cognitively challenging partially observable tasks. Not only did our approach solve problems that were unsolvable by A3C, ACER and PPO, but it did so in a relatively small number of training steps. RM learning provided the agent with memory, but more importantly the combination of RM learning and policy learning provided it with discrete reasoning capabilities that operated at a higher level of abstraction, while leveraging deep RL's ability to learn policies from low-level inputs.  This work leaves open many interesting questions relating to abstraction, observability, and properties of the language over which RMs are constructed. We believe that addressing these questions will push the boundary of partially observable RL problems that can be solved.

%% file: main.bbl
\providecommand{\Proceedings}{Proceedings\xspace}
  \providecommand{\International}{International\xspace}
  \providecommand{\Conference}{Conference\xspace}
  \providecommand{\Artificial}{Artificial\xspace}
  \providecommand{\Intelligence}{Intelligence\xspace}
  \providecommand{\AI}{Artificial Intelligence}
  \providecommand{\Scheduling}{Scheduling\xspace} \providecommand{\ofthe}{of
  the\xspace} \providecommand{\longshortnopar}[2]{#1}
  \providecommand{\longshort[2]}{#1 (#2)}
\begin{thebibliography}{56}
\expandafter\ifx\csname natexlab\endcsname\relax\def\natexlab#1{#1}\fi
\providecommand{\url}[1]{\texttt{#1}}
\providecommand{\href}[2]{#2}
\providecommand{\path}[1]{#1}
\providecommand{\DOIprefix}{doi:}
\providecommand{\ArXivprefix}{arXiv:}
\providecommand{\URLprefix}{URL: }
\providecommand{\Pubmedprefix}{pmid:}
\providecommand{\doi}[1]{\href{http://dx.doi.org/#1}{\path{#1}}}
\providecommand{\Pubmed}[1]{\href{pmid:#1}{\path{#1}}}
\providecommand{\bibinfo}[2]{#2}
\ifx\xfnm\relax \def\xfnm[#1]{\unskip,\space#1}\fi
\bibitem[{Aarts et~al.(2003)Aarts, Aarts and Lenstra}]{aarts2003local}
\bibinfo{author}{Aarts, E.}, \bibinfo{author}{Aarts, E.H.},
  \bibinfo{author}{Lenstra, J.K.}, \bibinfo{year}{2003}.
\newblock \bibinfo{title}{Local search in combinatorial optimization}.
\newblock \bibinfo{publisher}{Princeton University Press}.
\bibitem[{Andrychowicz et~al.(2018)Andrychowicz, Baker, Chociej, Jozefowicz,
  McGrew, Pachocki, Petron, Plappert, Powell, Ray
  et~al.}]{andrychowicz2018learning}
\bibinfo{author}{Andrychowicz, M.}, \bibinfo{author}{Baker, B.},
  \bibinfo{author}{Chociej, M.}, \bibinfo{author}{Jozefowicz, R.},
  \bibinfo{author}{McGrew, B.}, \bibinfo{author}{Pachocki, J.},
  \bibinfo{author}{Petron, A.}, \bibinfo{author}{Plappert, M.},
  \bibinfo{author}{Powell, G.}, \bibinfo{author}{Ray, A.}, et~al.,
  \bibinfo{year}{2018}.
\newblock \bibinfo{title}{Learning dexterous in-hand manipulation}.
\newblock \bibinfo{journal}{CoRR} \bibinfo{volume}{abs/1808.00177}.
\newblock \URLprefix \url{http://arxiv.org/abs/1808.00177}.
\bibitem[{Angluin(1987)}]{angluin1987learning}
\bibinfo{author}{Angluin, D.}, \bibinfo{year}{1987}.
\newblock \bibinfo{title}{Learning regular sets from queries and
  counterexamples}.
\newblock \bibinfo{journal}{Information and computation} \bibinfo{volume}{75},
  \bibinfo{pages}{87--106}.
\bibitem[{Angluin and Smith(1983)}]{angluin1983inductive}
\bibinfo{author}{Angluin, D.}, \bibinfo{author}{Smith, C.H.},
  \bibinfo{year}{1983}.
\newblock \bibinfo{title}{Inductive inference: Theory and methods}.
\newblock \bibinfo{journal}{ACM Computing Surveys (CSUR)} \bibinfo{volume}{15},
  \bibinfo{pages}{237--269}.
\bibitem[{Camacho et~al.(2019)Camacho, Toro~Icarte, Klassen, Valenzano and
  McIlraith}]{camamacho2019ijcai}
\bibinfo{author}{Camacho, A.}, \bibinfo{author}{Toro~Icarte, R.},
  \bibinfo{author}{Klassen, T.Q.}, \bibinfo{author}{Valenzano, R.},
  \bibinfo{author}{McIlraith, S.A.}, \bibinfo{year}{2019}.
\newblock \bibinfo{title}{{LTL} and beyond: Formal languages for reward
  function specification in reinforcement learning}, in:
  \bibinfo{booktitle}{\longshort{\Proceedings \ofthe 28th \International Joint
  \Conference on \AI{}}{IJCAI}}, pp. \bibinfo{pages}{6065--6073}.
\bibitem[{Cassandra et~al.(1994)Cassandra, Kaelbling and
  Littman}]{cassandra1994acting}
\bibinfo{author}{Cassandra, A.R.}, \bibinfo{author}{Kaelbling, L.P.},
  \bibinfo{author}{Littman, M.L.}, \bibinfo{year}{1994}.
\newblock \bibinfo{title}{Acting optimally in partially observable stochastic
  domains}, in: \bibinfo{booktitle}{\longshort{\Proceedings \ofthe 12th
  National \Conference on \AI{}}{AAAI}}, pp. \bibinfo{pages}{1023--1028}.
\bibitem[{{De Giacomo} et~al.(2020){De Giacomo}, Favorito, Iocchi, Patrizi and
  Ronca}]{DeGiacomo2020transducers}
\bibinfo{author}{{De Giacomo}, G.}, \bibinfo{author}{Favorito, M.},
  \bibinfo{author}{Iocchi, L.}, \bibinfo{author}{Patrizi, F.},
  \bibinfo{author}{Ronca, A.}, \bibinfo{year}{2020}.
\newblock \bibinfo{title}{Temporal logic monitoring rewards via transducers},
  in: \bibinfo{booktitle}{\longshort{\Proceedings \ofthe 17th \International
  \Conference on Knowledge Representation and Reasoning}{KR}}, pp.
  \bibinfo{pages}{860--870}.
\bibitem[{Doshi-Velez et~al.(2013)Doshi-Velez, Pfau, Wood and
  Roy}]{doshi2013bayesian}
\bibinfo{author}{Doshi-Velez, F.}, \bibinfo{author}{Pfau, D.},
  \bibinfo{author}{Wood, F.}, \bibinfo{author}{Roy, N.}, \bibinfo{year}{2013}.
\newblock \bibinfo{title}{Bayesian nonparametric methods for
  partially-observable reinforcement learning}.
\newblock \bibinfo{journal}{IEEE transactions on pattern analysis and machine
  intelligence} \bibinfo{volume}{37}, \bibinfo{pages}{394--407}.
\bibitem[{Dulac-Arnold et~al.(2021)Dulac-Arnold, Levine, Mankowitz, Li,
  Paduraru, Gowal and Hester}]{dulac2021challenges}
\bibinfo{author}{Dulac-Arnold, G.}, \bibinfo{author}{Levine, N.},
  \bibinfo{author}{Mankowitz, D.J.}, \bibinfo{author}{Li, J.},
  \bibinfo{author}{Paduraru, C.}, \bibinfo{author}{Gowal, S.},
  \bibinfo{author}{Hester, T.}, \bibinfo{year}{2021}.
\newblock \bibinfo{title}{Challenges of real-world reinforcement learning:
  definitions, benchmarks and analysis}.
\newblock \bibinfo{journal}{Machine Learning} \bibinfo{volume}{110},
  \bibinfo{pages}{1--50}.
\bibitem[{Dulac-Arnold et~al.(2019)Dulac-Arnold, Mankowitz and
  Hester}]{dulac2019challenges}
\bibinfo{author}{Dulac-Arnold, G.}, \bibinfo{author}{Mankowitz, D.},
  \bibinfo{author}{Hester, T.}, \bibinfo{year}{2019}.
\newblock \bibinfo{title}{Challenges of real-world reinforcement learning}.
\newblock \bibinfo{journal}{CoRR} \bibinfo{volume}{abs/1904.12901}.
\newblock \URLprefix \url{http://arxiv.org/abs/1904.12901}.
\bibitem[{Furelos-Blanco et~al.(2020)Furelos-Blanco, Law, Russo, Broda and
  Jonsson}]{furelos2020inductionAAAI}
\bibinfo{author}{Furelos-Blanco, D.}, \bibinfo{author}{Law, M.},
  \bibinfo{author}{Russo, A.}, \bibinfo{author}{Broda, K.},
  \bibinfo{author}{Jonsson, A.}, \bibinfo{year}{2020}.
\newblock \bibinfo{title}{Induction of subgoal automata for reinforcement
  learning.}, in: \bibinfo{booktitle}{\longshort{\Proceedings \ofthe 34th AAAI
  \Conference on \AI{}}{AAAI}}, pp. \bibinfo{pages}{3890--3897}.
\bibitem[{Gaon and Brafman(2020)}]{gaon2019reinforcement}
\bibinfo{author}{Gaon, M.}, \bibinfo{author}{Brafman, R.},
  \bibinfo{year}{2020}.
\newblock \bibinfo{title}{Reinforcement learning with non-{M}arkovian rewards},
  in: \bibinfo{booktitle}{\longshort{\Proceedings \ofthe 34th AAAI \Conference
  on \AI{}}{AAAI}}, pp. \bibinfo{pages}{3980--3987}.
\bibitem[{Ghavamzadeh et~al.(2015)Ghavamzadeh, Mannor, Pineau, Tamar
  et~al.}]{ghavamzadeh2015bayesian}
\bibinfo{author}{Ghavamzadeh, M.}, \bibinfo{author}{Mannor, S.},
  \bibinfo{author}{Pineau, J.}, \bibinfo{author}{Tamar, A.}, et~al.,
  \bibinfo{year}{2015}.
\newblock \bibinfo{title}{Bayesian reinforcement learning: A survey}.
\newblock \bibinfo{journal}{Foundations and Trends in Machine Learning}
  \bibinfo{volume}{8}, \bibinfo{pages}{359--483}.
\bibitem[{Giantamidis and Tripakis(2016)}]{giantamidis2016learning}
\bibinfo{author}{Giantamidis, G.}, \bibinfo{author}{Tripakis, S.},
  \bibinfo{year}{2016}.
\newblock \bibinfo{title}{Learning {M}oore machines from input-output traces},
  in: \bibinfo{booktitle}{\longshort{\Proceedings \ofthe 21st \International
  Symposium on Formal Methods}{FM}}, pp. \bibinfo{pages}{291--309}.
\bibitem[{Glover and Laguna(1998)}]{glover1998tabu}
\bibinfo{author}{Glover, F.}, \bibinfo{author}{Laguna, M.},
  \bibinfo{year}{1998}.
\newblock \bibinfo{title}{Tabu search}, in: \bibinfo{booktitle}{Handbook of
  combinatorial optimization}. \bibinfo{publisher}{Springer}, pp.
  \bibinfo{pages}{2093--2229}.
\bibitem[{{Gurobi Optimization, LLC}(2018)}]{gurobi}
\bibinfo{author}{{Gurobi Optimization, LLC}}, \bibinfo{year}{2018}.
\newblock \bibinfo{title}{Gurobi Optimizer Reference Manual}.
\newblock \URLprefix \url{http://www.gurobi.com}.
\bibitem[{Hasanbeig et~al.(2021)Hasanbeig, Jeppu, Abate, Melham and
  Kroening}]{hasanbeig2019deepsynth}
\bibinfo{author}{Hasanbeig, M.}, \bibinfo{author}{Jeppu, N.Y.},
  \bibinfo{author}{Abate, A.}, \bibinfo{author}{Melham, T.},
  \bibinfo{author}{Kroening, D.}, \bibinfo{year}{2021}.
\newblock \bibinfo{title}{Deepsynth: Automata synthesis for automatic task
  segmentation in deep reinforcement learning}.
\newblock \bibinfo{journal}{CoRR} \bibinfo{volume}{abs/1911.10244}.
\newblock \URLprefix \url{http://arxiv.org/abs/1911.10244}.
\bibitem[{Hausknecht and Stone(2015)}]{hausknecht2015deep}
\bibinfo{author}{Hausknecht, M.}, \bibinfo{author}{Stone, P.},
  \bibinfo{year}{2015}.
\newblock \bibinfo{title}{Deep recurrent q-learning for partially observable
  {MDP}s}, in: \bibinfo{booktitle}{AAAI Fall Symposium on Sequential Decision
  Making for Intelligent Agents (AAAI-SDMIA15)}.
\bibitem[{De~la Higuera(2010)}]{de2010grammatical}
\bibinfo{author}{De~la Higuera, C.}, \bibinfo{year}{2010}.
\newblock \bibinfo{title}{Grammatical inference: learning automata and
  grammars}.
\newblock \bibinfo{publisher}{Cambridge University Press}.
\bibitem[{Hung et~al.(2018)Hung, Lillicrap, Abramson, Wu, Mirza, Carnevale,
  Ahuja and Wayne}]{hung2018optimizing}
\bibinfo{author}{Hung, C.C.}, \bibinfo{author}{Lillicrap, T.},
  \bibinfo{author}{Abramson, J.}, \bibinfo{author}{Wu, Y.},
  \bibinfo{author}{Mirza, M.}, \bibinfo{author}{Carnevale, F.},
  \bibinfo{author}{Ahuja, A.}, \bibinfo{author}{Wayne, G.},
  \bibinfo{year}{2018}.
\newblock \bibinfo{title}{Optimizing agent behavior over long time scales by
  transporting value}.
\newblock \bibinfo{journal}{CoRR} \bibinfo{volume}{abs/1810.06721}.
\newblock \URLprefix \url{http://arxiv.org/abs/1810.06721}.
\bibitem[{IBM(2018)}]{CPOManual}
\bibinfo{author}{IBM}, \bibinfo{year}{2018}.
\newblock \bibinfo{title}{ILOG CP Optimizer 12.8 Manual}.
\bibitem[{Izadi and Precup(2005)}]{izadi2005using}
\bibinfo{author}{Izadi, M.T.}, \bibinfo{author}{Precup, D.},
  \bibinfo{year}{2005}.
\newblock \bibinfo{title}{Using rewards for belief state updates in partially
  observable {M}arkov decision processes}, in:
  \bibinfo{booktitle}{\longshort{\Proceedings \ofthe 16th European \Conference
  on Machine Learning}{ECML}}, pp. \bibinfo{pages}{593--600}.
\bibitem[{Jaderberg et~al.(2016)Jaderberg, Mnih, Czarnecki, Schaul, Leibo,
  Silver and Kavukcuoglu}]{jaderberg2016reinforcement}
\bibinfo{author}{Jaderberg, M.}, \bibinfo{author}{Mnih, V.},
  \bibinfo{author}{Czarnecki, W.M.}, \bibinfo{author}{Schaul, T.},
  \bibinfo{author}{Leibo, J.Z.}, \bibinfo{author}{Silver, D.},
  \bibinfo{author}{Kavukcuoglu, K.}, \bibinfo{year}{2016}.
\newblock \bibinfo{title}{Reinforcement learning with unsupervised auxiliary
  tasks}.
\newblock \bibinfo{journal}{CoRR} \bibinfo{volume}{abs/1611.05397}.
\newblock \URLprefix \url{http://arxiv.org/abs/1611.05397}.
\bibitem[{J{\"u}nger et~al.(2009)J{\"u}nger, Liebling, Naddef, Nemhauser,
  Pulleyblank, Reinelt, Rinaldi and Wolsey}]{junger200950}
\bibinfo{author}{J{\"u}nger, M.}, \bibinfo{author}{Liebling, T.M.},
  \bibinfo{author}{Naddef, D.}, \bibinfo{author}{Nemhauser, G.L.},
  \bibinfo{author}{Pulleyblank, W.R.}, \bibinfo{author}{Reinelt, G.},
  \bibinfo{author}{Rinaldi, G.}, \bibinfo{author}{Wolsey, L.A.},
  \bibinfo{year}{2009}.
\newblock \bibinfo{title}{50 Years of integer programming 1958-2008: From the
  early years to the state-of-the-art}.
\newblock \bibinfo{publisher}{Springer Science \& Business Media}.
\bibitem[{Kaelbling et~al.(1996)Kaelbling, Littman and
  Moore}]{kaelbling1996reinforcement}
\bibinfo{author}{Kaelbling, L.P.}, \bibinfo{author}{Littman, M.L.},
  \bibinfo{author}{Moore, A.W.}, \bibinfo{year}{1996}.
\newblock \bibinfo{title}{Reinforcement learning: A survey}.
\newblock \bibinfo{journal}{Journal of artificial intelligence research}
  \bibinfo{volume}{4}, \bibinfo{pages}{237--285}.
\bibitem[{Khan et~al.(2017)Khan, Zhang, Atanasov, Karydis, Kumar and
  Lee}]{khan2017memory}
\bibinfo{author}{Khan, A.}, \bibinfo{author}{Zhang, C.},
  \bibinfo{author}{Atanasov, N.}, \bibinfo{author}{Karydis, K.},
  \bibinfo{author}{Kumar, V.}, \bibinfo{author}{Lee, D.D.},
  \bibinfo{year}{2017}.
\newblock \bibinfo{title}{Memory augmented control networks}.
\newblock \bibinfo{journal}{CoRR} \bibinfo{volume}{abs/1709.05706}.
\newblock \URLprefix \url{http://arxiv.org/abs/1709.05706}.
\bibitem[{Littman(1993)}]{littman1993optimization}
\bibinfo{author}{Littman, M.L.}, \bibinfo{year}{1993}.
\newblock \bibinfo{title}{An optimization-based categorization of reinforcement
  learning environments}, in: \bibinfo{booktitle}{From Animals to Animats 2:
  Proceedings of the Second International Conference on Simulation of Adaptive
  Behavior}, pp. \bibinfo{pages}{262--270}.
\bibitem[{Littman et~al.(2002)Littman, Sutton and
  Singh}]{littman2002predictive}
\bibinfo{author}{Littman, M.L.}, \bibinfo{author}{Sutton, R.S.},
  \bibinfo{author}{Singh, S.}, \bibinfo{year}{2002}.
\newblock \bibinfo{title}{Predictive representations of state}, in:
  \bibinfo{booktitle}{\longshort{\Proceedings \ofthe 15th \Conference on
  Advances in Neural Information Processing Systems}{NIPS}}, pp.
  \bibinfo{pages}{1555--1561}.
\bibitem[{Mahmud(2010)}]{mahmud2010constructing}
\bibinfo{author}{Mahmud, M.}, \bibinfo{year}{2010}.
\newblock \bibinfo{title}{Constructing states for reinforcement learning}, in:
  \bibinfo{booktitle}{\longshort{\Proceedings \ofthe 27th \International
  \Conference on Machine Learning}{ICML}}, pp. \bibinfo{pages}{727--734}.
\bibitem[{Memarian et~al.(2020)Memarian, Xu, Wu, Wen and
  Topcu}]{memarian2020active}
\bibinfo{author}{Memarian, F.}, \bibinfo{author}{Xu, Z.}, \bibinfo{author}{Wu,
  B.}, \bibinfo{author}{Wen, M.}, \bibinfo{author}{Topcu, U.},
  \bibinfo{year}{2020}.
\newblock \bibinfo{title}{Active task-inference-guided deep inverse
  reinforcement learning}, in: \bibinfo{booktitle}{\longshort{\Proceedings
  \ofthe 59th IEEE \Conference on on Decision and Control}{CDC}}, pp.
  \bibinfo{pages}{1932--1938}.
\bibitem[{Meuleau et~al.(1999)Meuleau, Peshkin, Kim and
  Kaelbling}]{meuleau1999learning}
\bibinfo{author}{Meuleau, N.}, \bibinfo{author}{Peshkin, L.},
  \bibinfo{author}{Kim, K.E.}, \bibinfo{author}{Kaelbling, L.P.},
  \bibinfo{year}{1999}.
\newblock \bibinfo{title}{Learning finite-state controllers for partially
  observable environments}, in: \bibinfo{booktitle}{\longshort{\Proceedings
  \ofthe 15th \Conference on Uncertainty in Artificial Intelligence}{UAI}}, pp.
  \bibinfo{pages}{427--436}.
\bibitem[{Mnih et~al.(2016)Mnih, Badia, Mirza, Graves, Lillicrap, Harley,
  Silver and Kavukcuoglu}]{mnih2016asynchronous}
\bibinfo{author}{Mnih, V.}, \bibinfo{author}{Badia, A.P.},
  \bibinfo{author}{Mirza, M.}, \bibinfo{author}{Graves, A.},
  \bibinfo{author}{Lillicrap, T.}, \bibinfo{author}{Harley, T.},
  \bibinfo{author}{Silver, D.}, \bibinfo{author}{Kavukcuoglu, K.},
  \bibinfo{year}{2016}.
\newblock \bibinfo{title}{Asynchronous methods for deep reinforcement
  learning}, in: \bibinfo{booktitle}{\longshort{\Proceedings \ofthe 33rd
  \International \Conference on Machine Learning}{ICML}}, pp.
  \bibinfo{pages}{1928--1937}.
\bibitem[{Mnih et~al.(2015)Mnih, Kavukcuoglu, Silver, Rusu, Veness, Bellemare,
  Graves, Riedmiller, Fidjeland, Ostrovski et~al.}]{mnih2015human}
\bibinfo{author}{Mnih, V.}, \bibinfo{author}{Kavukcuoglu, K.},
  \bibinfo{author}{Silver, D.}, \bibinfo{author}{Rusu, A.A.},
  \bibinfo{author}{Veness, J.}, \bibinfo{author}{Bellemare, M.G.},
  \bibinfo{author}{Graves, A.}, \bibinfo{author}{Riedmiller, M.},
  \bibinfo{author}{Fidjeland, A.K.}, \bibinfo{author}{Ostrovski, G.}, et~al.,
  \bibinfo{year}{2015}.
\newblock \bibinfo{title}{Human-level control through deep reinforcement
  learning}.
\newblock \bibinfo{journal}{Nature} \bibinfo{volume}{518},
  \bibinfo{pages}{529--533}.
\bibitem[{Neider et~al.(2021)Neider, Gaglione, Gavran, Topcu, Wu and
  Xu}]{neider2021advice}
\bibinfo{author}{Neider, D.}, \bibinfo{author}{Gaglione, J.R.},
  \bibinfo{author}{Gavran, I.}, \bibinfo{author}{Topcu, U.},
  \bibinfo{author}{Wu, B.}, \bibinfo{author}{Xu, Z.}, \bibinfo{year}{2021}.
\newblock \bibinfo{title}{Advice-guided reinforcement learning in a
  non-{M}arkovian environment}, in: \bibinfo{booktitle}{\longshort{\Proceedings
  \ofthe 35th AAAI \Conference on \AI{}}{AAAI}}, pp.
  \bibinfo{pages}{9073--9080}.
\bibitem[{Oh et~al.(2016)Oh, Chockalingam, Singh and Lee}]{oh2016control}
\bibinfo{author}{Oh, J.}, \bibinfo{author}{Chockalingam, V.},
  \bibinfo{author}{Singh, S.}, \bibinfo{author}{Lee, H.}, \bibinfo{year}{2016}.
\newblock \bibinfo{title}{Control of memory, active perception, and action in
  minecraft}, in: \bibinfo{booktitle}{\longshort{\Proceedings \ofthe 33rd
  \International \Conference on Machine Learning}{ICML}}, pp.
  \bibinfo{pages}{2790--2799}.
\bibitem[{Peshkin et~al.(1999)Peshkin, Meuleau and
  Kaelbling}]{peshkin1999learning}
\bibinfo{author}{Peshkin, L.}, \bibinfo{author}{Meuleau, N.},
  \bibinfo{author}{Kaelbling, L.P.}, \bibinfo{year}{1999}.
\newblock \bibinfo{title}{Learning policies with external memory}, in:
  \bibinfo{booktitle}{\longshort{\Proceedings \ofthe 16th \International
  \Conference on Machine Learning}{ICML}}, pp. \bibinfo{pages}{307--314}.
\bibitem[{Pisinger and Ropke(2010)}]{pisinger2010large}
\bibinfo{author}{Pisinger, D.}, \bibinfo{author}{Ropke, S.},
  \bibinfo{year}{2010}.
\newblock \bibinfo{title}{Large neighborhood search}, in:
  \bibinfo{booktitle}{Handbook of metaheuristics}.
  \bibinfo{publisher}{Springer}, pp. \bibinfo{pages}{399--419}.
\bibitem[{Poupart and Vlassis(2008)}]{poupart2008model}
\bibinfo{author}{Poupart, P.}, \bibinfo{author}{Vlassis, N.},
  \bibinfo{year}{2008}.
\newblock \bibinfo{title}{Model-based {B}ayesian reinforcement learning in
  partially observable domains}, in:
  \bibinfo{booktitle}{\longshort{\Proceedings \ofthe 10th \International
  Symposium on Artificial Intelligence and Mathematics}{ISAIM}}, pp.
  \bibinfo{pages}{1--2}.
\bibitem[{Rens et~al.(2020)Rens, Raskin, Reynouad and Marra}]{rens2020online}
\bibinfo{author}{Rens, G.}, \bibinfo{author}{Raskin, J.F.},
  \bibinfo{author}{Reynouad, R.}, \bibinfo{author}{Marra, G.},
  \bibinfo{year}{2020}.
\newblock \bibinfo{title}{Online learning of non-{M}arkovian reward models}.
\newblock \bibinfo{journal}{CoRR} \bibinfo{volume}{abs/2009.12600}.
\newblock \URLprefix \url{https://arxiv.org/abs/2009.12600}.
\bibitem[{Rossi et~al.(2006)Rossi, Van~Beek and Walsh}]{rossi2006handbook}
\bibinfo{author}{Rossi, F.}, \bibinfo{author}{Van~Beek, P.},
  \bibinfo{author}{Walsh, T.}, \bibinfo{year}{2006}.
\newblock \bibinfo{title}{Handbook of constraint programming}.
\newblock \bibinfo{publisher}{Elsevier}.
\bibitem[{Schulman et~al.(2017)Schulman, Wolski, Dhariwal, Radford and
  Klimov}]{schulman2017proximal}
\bibinfo{author}{Schulman, J.}, \bibinfo{author}{Wolski, F.},
  \bibinfo{author}{Dhariwal, P.}, \bibinfo{author}{Radford, A.},
  \bibinfo{author}{Klimov, O.}, \bibinfo{year}{2017}.
\newblock \bibinfo{title}{Proximal policy optimization algorithms}.
\newblock \bibinfo{journal}{CoRR} \bibinfo{volume}{abs/1707.06347}.
\newblock \URLprefix \url{http://arxiv.org/abs/1707.06347}.
\bibitem[{Shvo et~al.(2021)Shvo, Li, {Toro Icarte} and
  McIlraith}]{shvo2021interpretable}
\bibinfo{author}{Shvo, M.}, \bibinfo{author}{Li, A.C.}, \bibinfo{author}{{Toro
  Icarte}, R.}, \bibinfo{author}{McIlraith, S.A.}, \bibinfo{year}{2021}.
\newblock \bibinfo{title}{Interpretable sequence classification via discrete
  optimization}, in: \bibinfo{booktitle}{\longshort{\Proceedings \ofthe 35th
  AAAI \Conference on \AI{}}{AAAI}}, pp. \bibinfo{pages}{9647--9656}.
\bibitem[{Silver et~al.(2017)Silver, Schrittwieser, Simonyan, Antonoglou,
  Huang, Guez, Hubert, Baker, Lai, Bolton et~al.}]{silver2017mastering}
\bibinfo{author}{Silver, D.}, \bibinfo{author}{Schrittwieser, J.},
  \bibinfo{author}{Simonyan, K.}, \bibinfo{author}{Antonoglou, I.},
  \bibinfo{author}{Huang, A.}, \bibinfo{author}{Guez, A.},
  \bibinfo{author}{Hubert, T.}, \bibinfo{author}{Baker, L.},
  \bibinfo{author}{Lai, M.}, \bibinfo{author}{Bolton, A.}, et~al.,
  \bibinfo{year}{2017}.
\newblock \bibinfo{title}{Mastering the game of {Go} without human knowledge}.
\newblock \bibinfo{journal}{Nature} \bibinfo{volume}{550},
  \bibinfo{pages}{354}.
\bibitem[{Singh et~al.(1994)Singh, Jaakkola and Jordan}]{singh1994learning}
\bibinfo{author}{Singh, S.P.}, \bibinfo{author}{Jaakkola, T.},
  \bibinfo{author}{Jordan, M.I.}, \bibinfo{year}{1994}.
\newblock \bibinfo{title}{Learning without state-estimation in partially
  observable {M}arkovian decision processes}, in: \bibinfo{booktitle}{Machine
  Learning Proceedings 1994}. \bibinfo{publisher}{Elsevier}, pp.
  \bibinfo{pages}{284--292}.
\bibitem[{Sutton and Barto(2018)}]{sutton2018reinforcement}
\bibinfo{author}{Sutton, R.S.}, \bibinfo{author}{Barto, A.G.},
  \bibinfo{year}{2018}.
\newblock \bibinfo{title}{Reinforcement learning: An introduction}.
\newblock \bibinfo{publisher}{MIT press}.
\bibitem[{{Toro Icarte} et~al.(2018){Toro Icarte}, Klassen, Valenzano and
  McIlraith}]{icml2018rms}
\bibinfo{author}{{Toro Icarte}, R.}, \bibinfo{author}{Klassen, T.Q.},
  \bibinfo{author}{Valenzano, R.}, \bibinfo{author}{McIlraith, S.A.},
  \bibinfo{year}{2018}.
\newblock \bibinfo{title}{Using reward machines for high-level task
  specification and decomposition in reinforcement learning}, in:
  \bibinfo{booktitle}{\longshort{\Proceedings \ofthe 35th \International
  \Conference on Machine Learning}{ICML}}, pp. \bibinfo{pages}{2112--2121}.
\bibitem[{{Toro Icarte} et~al.(2020a){Toro Icarte}, Klassen, Valenzano and
  McIlraith}]{icarte2020reward}
\bibinfo{author}{{Toro Icarte}, R.}, \bibinfo{author}{Klassen, T.Q.},
  \bibinfo{author}{Valenzano, R.}, \bibinfo{author}{McIlraith, S.A.},
  \bibinfo{year}{2020}a.
\newblock \bibinfo{title}{Reward machines: Exploiting reward function structure
  in reinforcement learning}.
\newblock \bibinfo{journal}{CoRR} \bibinfo{volume}{abs/2010.03950}.
\newblock \URLprefix \url{https://arxiv.org/abs/2010.03950}.
\bibitem[{{Toro Icarte} et~al.(2020b){Toro Icarte}, Valenzano, Klassen,
  Christoffersen, massoud Farahmand and McIlraith}]{icarte2020act}
\bibinfo{author}{{Toro Icarte}, R.}, \bibinfo{author}{Valenzano, R.},
  \bibinfo{author}{Klassen, T.Q.}, \bibinfo{author}{Christoffersen, P.},
  \bibinfo{author}{massoud Farahmand, A.}, \bibinfo{author}{McIlraith, S.A.},
  \bibinfo{year}{2020}b.
\newblock \bibinfo{title}{The act of remembering: a study in partially
  observable reinforcement learning}.
\newblock \bibinfo{journal}{CoRR} \bibinfo{volume}{abs/2010.01753}.
\newblock \URLprefix \url{http://arxiv.org/abs/2010.01753}.
\bibitem[{{Toro Icarte} et~al.(2019){Toro Icarte}, Waldie, Klassen, Valenzano,
  Castro and McIlraith}]{tor-etal-neurips19}
\bibinfo{author}{{Toro Icarte}, R.}, \bibinfo{author}{Waldie, E.},
  \bibinfo{author}{Klassen, T.Q.}, \bibinfo{author}{Valenzano, R.},
  \bibinfo{author}{Castro, M.P.}, \bibinfo{author}{McIlraith, S.A.},
  \bibinfo{year}{2019}.
\newblock \bibinfo{title}{Learning reward machines for partially observable
  reinforcement learning}, in: \bibinfo{booktitle}{\longshort{\Proceedings
  \ofthe 32nd \Conference on Advances in Neural Information Processing
  Systems}{NeurIPS}}, pp. \bibinfo{pages}{15497--15508}.
\bibitem[{Van~Hasselt et~al.(2016)Van~Hasselt, Guez and Silver}]{van2016deep}
\bibinfo{author}{Van~Hasselt, H.}, \bibinfo{author}{Guez, A.},
  \bibinfo{author}{Silver, D.}, \bibinfo{year}{2016}.
\newblock \bibinfo{title}{Deep reinforcement learning with {D}ouble
  {Q}-learning}, in: \bibinfo{booktitle}{\longshort{\Proceedings \ofthe 30th
  AAAI \Conference on \AI{}}{AAAI}}, pp. \bibinfo{pages}{2094--2100}.
\bibitem[{Wang et~al.(2016)Wang, Bapst, Heess, Mnih, Munos, Kavukcuoglu and
  de~Freitas}]{wang2016sample}
\bibinfo{author}{Wang, Z.}, \bibinfo{author}{Bapst, V.},
  \bibinfo{author}{Heess, N.}, \bibinfo{author}{Mnih, V.},
  \bibinfo{author}{Munos, R.}, \bibinfo{author}{Kavukcuoglu, K.},
  \bibinfo{author}{de~Freitas, N.}, \bibinfo{year}{2016}.
\newblock \bibinfo{title}{Sample efficient actor-critic with experience
  replay}.
\newblock \bibinfo{journal}{CoRR} \bibinfo{volume}{abs/1611.01224}.
\newblock \URLprefix \url{http://arxiv.org/abs/1611.01224}.
\bibitem[{Watkins and Dayan(1992)}]{watkins1992q}
\bibinfo{author}{Watkins, C.J.C.H.}, \bibinfo{author}{Dayan, P.},
  \bibinfo{year}{1992}.
\newblock \bibinfo{title}{Q-learning}.
\newblock \bibinfo{journal}{Machine learning} \bibinfo{volume}{8},
  \bibinfo{pages}{279--292}.
\bibitem[{Xu et~al.(2020a)Xu, Gavran, Ahmad, Majumdar, Neider, Topcu and
  Wu}]{xu2020joint}
\bibinfo{author}{Xu, Z.}, \bibinfo{author}{Gavran, I.}, \bibinfo{author}{Ahmad,
  Y.}, \bibinfo{author}{Majumdar, R.}, \bibinfo{author}{Neider, D.},
  \bibinfo{author}{Topcu, U.}, \bibinfo{author}{Wu, B.}, \bibinfo{year}{2020}a.
\newblock \bibinfo{title}{Joint inference of reward machines and policies for
  reinforcement learning}, in: \bibinfo{booktitle}{\longshort{\Proceedings
  \ofthe 30th \International \Conference on Automated Planning and
  \Scheduling}{ICAPS}}, pp. \bibinfo{pages}{590--598}.
\bibitem[{Xu et~al.(2020b)Xu, Wu, Neider and Topcu}]{xu2020active}
\bibinfo{author}{Xu, Z.}, \bibinfo{author}{Wu, B.}, \bibinfo{author}{Neider,
  D.}, \bibinfo{author}{Topcu, U.}, \bibinfo{year}{2020}b.
\newblock \bibinfo{title}{Active finite reward automaton inference and
  reinforcement learning using queries and counterexamples}.
\newblock \bibinfo{journal}{CoRR} \bibinfo{volume}{abs/2006.15714}.
\newblock \URLprefix \url{https://arxiv.org/abs/2006.15714}.
\bibitem[{Zeng et~al.(1993)Zeng, Goodman and Smyth}]{zeng1993learning}
\bibinfo{author}{Zeng, Z.}, \bibinfo{author}{Goodman, R.M.},
  \bibinfo{author}{Smyth, P.}, \bibinfo{year}{1993}.
\newblock \bibinfo{title}{Learning finite state machines with self-clustering
  recurrent networks}.
\newblock \bibinfo{journal}{Neural Computation} \bibinfo{volume}{5},
  \bibinfo{pages}{976--990}.
\bibitem[{Zhang et~al.(2016)Zhang, McCarthy, Finn, Levine and
  Abbeel}]{zhang2016learning}
\bibinfo{author}{Zhang, M.}, \bibinfo{author}{McCarthy, Z.},
  \bibinfo{author}{Finn, C.}, \bibinfo{author}{Levine, S.},
  \bibinfo{author}{Abbeel, P.}, \bibinfo{year}{2016}.
\newblock \bibinfo{title}{Learning deep neural network policies with continuous
  memory states}, in: \bibinfo{booktitle}{\longshort{\Proceedings \ofthe 2016
  IEEE \International \Conference on Robotics and Automation}{ICRA}}, pp.
  \bibinfo{pages}{520--527}.

\end{thebibliography}
